\documentclass[11pt]{amsart}
\usepackage{hyperref}
\usepackage{url} 
\usepackage[sort&compress,numbers]{natbib}
\usepackage{amsmath,amssymb,amsthm,enumerate,bbm,multirow}
\usepackage[headings]{fullpage}
\usepackage[foot]{amsaddr}

\usepackage[ruled,vlined,norelsize]{algorithm2e}
\newtheorem{theorem}{Theorem}[section]
\newtheorem{proposition}[theorem]{Proposition}

\newtheorem{lemma}[theorem]{Lemma}
\theoremstyle{definition}

\theoremstyle{remark}
\newtheorem{remark}[theorem]{Remark}

\title[Sparse regret minimization]{Gains and Losses are Fundamentally
  Different in Regret Minimization: The Sparse Case}

\author[Kwon]{Joon Kwon$^\dagger$}
\address{$\dagger$ \textnormal{Institut de Math\'{e}matiques de Jussieu\\Universit\'{e}
  Pierre-et-Marie-Curie\\Paris, France.}}
\email{joon.kwon@ens-lyon.org}

\author[Perchet]{Vianney Perchet$^\ddagger$}
\address{$\ddagger$ \textnormal{INRIA \& Laboratoire de Probabilit\'{e}s et Mod\`{e}les Al\'{e}atoires,
  Universit\'{e} Paris-Diderot, Paris, France.}}
\email{vianney.perchet@normalesup.org}

\keywords{online optimization; regret minimization; adversarial; sparse; bandit}

\renewcommand{\leq}{\leqslant}
\renewcommand{\geq}{\geqslant}

\begin{document}
\maketitle
\begin{abstract}
  
  We demonstrate that, in the classical non-stochastic regret minimization problem with $d$
  decisions,  gains and losses to be respectively  maximized or minimized are fundamentally different. Indeed, by considering the additional sparsity assumption  (at each
  stage, at most $s$ decisions incur a nonzero outcome), we  
derive optimal regret bounds of
  different orders. Specifically, with gains, we obtain an optimal regret guarantee after $T$
  stages of order $\sqrt{T\log s}$, so the classical
  dependency in the dimension is replaced by the sparsity
  size. With losses, we
  provide matching upper and lower bounds of order
  $\sqrt{Ts\log(d)/d}$, which is decreasing in $d$. Eventually, we also study
  the bandit setting, and obtain an upper bound of order
  $\sqrt{Ts\log (d/s)}$ when outcomes are losses. This bound is proven
  to be optimal up to the logarithmic factor $\sqrt{\log(d/s)}$.
\end{abstract}

\section{Introduction}

We consider the  classical problem of regret
minimization \cite{hannan1957approximation} that has been well
developed during the last decade
\cite{cesa2006prediction,rakhlinlecture,bubeck2011introduction,shalev2011online,hazan201210,bubeck2012regret}. We
recall that in this sequential decision problem, a decision maker (or
agent, player, algorithm, strategy, policy, depending on the context) chooses at each stage a decision in a finite set
(that we write as $[d]:=\{1,\ldots,d\}$) and obtains as an
\emph{outcome} a real number in $[0,1]$. We specifically chose the
word \emph{outcome}, as opposed to \emph{gain} or \emph{loss}, as our
results show that there exists a fundamental discrepancy between these
two concepts.

The criterion used to evaluate the policy of the decision maker is
 the \emph{regret}, i.e., the difference between the
cumulative performance of the best stationary policy (that always
picks a given action $i \in [d]$) and the cumulative performance of
the policy of the decision maker.

We focus here on the \emph{non-stochastic} framework, where no
assumption (apart from boundedness) is made on the sequence of
possible outcomes. In particular, they are not i.i.d.\ and we can even
assume, as usual, that they depend on the past choices of the
decision maker. This broad setup, sometimes referred to as
\emph{individual sequences} (since a policy must be good against \emph{any}
sequence of possible outcomes) incorporates prediction with
expert advice \cite{cesa2006prediction}, data with time-evolving
laws, etc. Perhaps the most fundamental results in this setup are the upper bound of order $\sqrt{T\log d}$
achieved by the Exponential Weight Algorithm
\cite{littlestone1994weighted,vovk1990aggregating,cesa1997analysis,auer2002adaptive}
and the asymptotic lower bound of the same order \cite{cesa1997use}. This
general bound is the same whether outcomes are gains in $[0,1]$ (in
which case, the objective is to maximize the cumulative sum of gains) or
losses in $[0,1]$ (where the decision maker aims at minimizing the 
cumulative sum). Indeed, a loss $\ell$ can easily be turned into gain
$g$ by defining $g:=1 -\ell$, the regret being invariant under this
transformation.

This idea does not apply anymore with 
 structural assumption.  For instance, consider the framework where the outcomes
 are limited to \emph{$s$-sparse vectors}, i.e. vectors that have at most $s$ nonzero coordinates.
 The coordinates which are nonzero may
 change arbitrarily over time.
 In this framework, the aforementioned
transformation does not preserve the sparsity assumption. Indeed, if $(\ell_1,\ldots,
\ell_d)$ is a  $s$-sparse loss vector, the corresponding gain vector $(1-\ell_1,\dots,1-\ell_d)$
may even have full support.
Consequently, results for loss vectors do not apply directly to sparse gains, and
vice versa. It turns out that both setups are fundamentally
different.

The sparsity assumption is actually quite natural in learning and have also
received some attention in online learning
\cite{gerchinovitz2013sparsity,carpentier2012bandit,abbasi2012online,djolonga2013high}. In the case of gains, it reflects the fact
that the problem has some hidden structure and that many
options are irrelevant. For instance, in the canonical
click-through-rate example, a website displays an ad and gets rewarded
if the user clicks on it;  we can safely assume that there are only a
small number of ads on which a user would click.

The sparse scenario can also be seen through the scope of prediction
with experts. Given a finite set of expert, we call the \emph{winner of a
stage} the expert with the highest revenue (or the smallest loss);
ties are broken arbitrarily. And the objective would be to win as many
stages as possible. The $s$-sparse setting would represent the case
where $s$ experts are designated as winners (or, non-loser) at each stage.

In the case of losses, the sparsity assumption is motivated by
situations where rare failures might happen at each stage, and the decision maker wants to avoid them.
For instance,  in network routing problems, it could be assumes that only a small number of paths would
lose packets as a result of a single, rare, server failure. Or a learner could have access to a finite number of 
classification algorithms that perform ideally most of the time; unfortunately, some of them makes mistakes on some examples and the learner would like to prevent that.  The general setup is therefore a number of algorithms/experts/actions that
mostly perform well (i.e., find the
correct path, classify correctly, optimize correctly some target function, etc.); however, at each time instance, there are rare
mistakes/accidents and the objective would be to find the action/algorithm
that has the smallest number (or probability in the stochastic case) of failures.

% \textbf{Summary of results.}
\subsection{Summary of Results}

We investigate  regret minimization scenarios both when outcomes are
gains on the one hand, and losses on the other hand. We recall that our  objectives are to prove that
they are fundamentally different by exhibiting rates of convergence of different order.

When outcomes are gains, we construct an algorithm based on the Online
Mirror Descent family
\cite{shalev2007online,shalev2011online,bubeck2011introduction}. By
choosing a regularizer based on the $\ell^p$ norm, and then tuning the
parameter $p$ as a function of $s$, we get in
Theorem~\ref{thm:full-info-gains-m-2} a regret bound of order $\sqrt{T\log s}$,
which has the interesting property of being independent of the number
of decisions $d$. This bound is trivially optimal, up to the constant.

If outcomes are losses instead of gains, although the previous
analysis remains valid, a much better bound can be obtained. We build
upon a regret bound for the Exponential Weight Algorithm
\cite{littlestone1994weighted,freund1997decision} and we manage to get
in Theorem~\ref{thm:upper bound-losses-full-info} a regret bound of
order $\sqrt{\frac{Ts\log d}{d}}$,
which is \emph{decreasing} in $d$, for a given
$s$. A nontrivial matching lower bound is established in
Theorem~\ref{thm:lower bound-losses}.

Both of these algorithms need to be tuned as a function of $s$. In
Theorem~\ref{thm:adative-losses} and Theorem~\ref{thm:adaptive-gains},
we construct algorithms which essentially achieve the same regret
bounds without prior knowledge of $s$, by adapting over time to the
sparsity level of past outcome vectors, using an adapted version of the doubling trick.

Finally, we investigate the bandit setting, where the only
feedback available to the decision maker is the outcome of his
decisions (and, not
the outcome of all possible decisions). In the case of losses
 we obtain in Theorem~\ref{thm:bandit-upper bound-losses} an
upper bound of order $\sqrt{Ts\log (d/s)}$,
using the Greedy Online Mirror Descent family of algorithms
\cite{audibert2009minimax,audibert2013regret,bubeck2011introduction}.
This bound is proven to be optimal up to a logarithmic factor, as
Theorem~\ref{thm:bandit-losses-lowerbound} establishes a lower bound
of order $\sqrt{Ts}$.

The rates of convergence achieved by our algorithms are summarized in Figure~\ref{fig:tableau}.

\renewcommand{\arraystretch}{2}
\begin{figure}
  \centering
  \begin{tabular}{r|c|c|c|c|}
    \cline{2-5}
    &\multicolumn{2}{|c|}{Full information}&\multicolumn{2}{|c|}{Bandit}\\
    \cline{2-5}
    &Gains&Losses&Gains&Losses\\
    \hline
    \multicolumn{1}{|r|}{Upper bound}&\multirow{2}{*}{$\sqrt{T\log s}$}&\multirow{2}{*}{$\sqrt{Ts\frac{\log d}{d}}$}&$\sqrt{Td}$&$\sqrt{Ts\log \frac{d}{s}}$\\
    \cline{1-1}\cline{4-5}
    \multicolumn{1}{|r|}{Lower bound}&&&$\sqrt{Ts}$&$\sqrt{Ts}$\\
    \hline
  \end{tabular}
  \caption{Summary of upper and lower bounds.}
  \label{fig:tableau}
\end{figure}

\subsection{General Model and Notation}
\label{sec:framework-notation}

We recall the  classical non-stochastic regret minimization problem. 
At each time instance $t \geq  1$, the decision maker chooses a decision $d_t$
in the finite set $[d]=\{1,\ldots,d\}$, possibly at random,
accordingly to $x_t \in \Delta_d$, where 
\[ \Delta_d= \left\{ x=(x^{(1)},\ldots,x^{(d)})\in \mathbb{R}^d_+\,\middle|\,\sum_{i=1}^dx^{(i)}=1 \right\} \]
is the the set of probability distributions over $[d]$.
Nature then reveals an outcome
vector $\omega_t \in [0,1]^d$ and the decision maker receives
$\omega_t^{(d_t)} \in [0,1]$. As outcomes are bounded, we can easily
replace  $\omega_t^{(d_t)}$ by its expectation that we denote by
$\langle \omega_t,x_t\rangle$. Indeed,  Hoeffding-Azuma concentration inequality will imply that all the results we will state in expectation hold with  high probability.

 Given a time horizon $T \geqslant 1$, the objective of the decision
 maker is to minimize his regret, whose definition depends on whether
 outcomes are \emph{gains} or \emph{losses}. In the case of gains
 (resp. losses), the notation $\omega_t$ is then changed to $g_t$
 (resp. $\ell_t$) and the regret is:
\[ R_T=\max_{i \in [d]}\sum_{t=1}^Tg_t^{(i)} -\sum_{t=1}^T \langle g_t, x_t\rangle\quad \left( \text{resp.}\ R_T=\sum_{t=1}^T \langle \ell_t,x_t\rangle-\min_{i \in [d]}\sum_{t=1}^T\ell_t^{(i)} \right). \]
In both cases, the well-known Exponential Weight
Algorithm  guarantees a bound on the
regret of order $\sqrt{T\log d}$. Moreover, this bound cannot be
improved in general as it matches a lower bound. 

\bigskip

We shall consider an additional structural
assumption on the outcomes, namely that $\omega_t$ is $s$-sparse in
the sense that $\|\omega_t\|_0 \leq s$, i.e., the number of nonzero components
of $\omega_t$ is less than $s$, where $s$ is a fixed known
parameter. The set of components which are nonzero is not fixed nor known, and
may change arbitrarily over time.

We aim at proving that it is then possible to drastically improve the
previously mentioned guarantee of order $\sqrt{T\log d}$ and that
losses and gains are two fundamentally different settings with minimax
regrets of different orders.

\section{When Outcomes are Gains to be Maximized}

\subsection{Online Mirror Descent Algorithms}
\label{sec:online-mirr-desc}
We quickly present the general Online Mirror Descent algorithm \cite{shalev2011online,bubeck2011introduction,bubeck2012regret,kwon2014continuous} and
state the regret bound it incurs; it will be used as a key element in 
Theorem~\ref{thm:full-info-gains-m-2}.
 
A convex function $h:\mathbb{R}^d\to \mathbb{R}\cup\{+\infty\}$ is called a \emph{regularizer} on $\Delta_d$ if $h$ is strictly convex
and continuous on its domain $\Delta_d$, and $h(x)=+\infty$ outside $\Delta_d$. Denote $\delta_h=\max_{\Delta_d}h-\min_{\Delta_d}h$ and $h^*:\mathbb{R}^d\to \mathbb{R}^d$ the Legendre-Fenchel
transform of $h$:
\[ h^*(y)=\sup_{x\in \mathbb{R}^d}\left\{ \left< y , x \right>
    -h(x) \right\},\quad y\in \mathbb{R}^d,  \]
which is differentiable since $h$ is strictly convex. For all $y\in
\mathbb{R}^d$, it holds that $\nabla h^*(y)\in \Delta_d$.

Let $\eta \in \mathbb{R}$ be a parameter to be tuned. The Online Mirror Descent Algorithm
associated with regularizer $h$ and parameter $\eta$ is defined by:
\[ x_t=\nabla h^*\left( \eta\sum_{k=1}^{t-1}\omega_k \right),\quad t \geq 1,  \]
where $\omega_t \in [0,1]^d$ denote the  vector of outcomes and $x_t$ the probability distribution chosen at stage $t$. The specific choice $h(x)=\sum_{i=1}^dx^{(i)}\log x^{(i)}$ for $x=(x^{(1)},\ldots,x^{(d)})\in \Delta_d$ (and
$h(x)=+\infty$ otherwise) gives the celebrated Exponential Weight
Algorithm, which can we written explicitly, component by component:
\[ x_t^{(i)}=\frac{\exp \left( \eta\sum_{k=1}^{t-1}\omega_k^{(i)} \right)}{\sum_{j=1}^d\exp \left( \eta\sum_{k=1}^{t-1}\omega_k^{(j)} \right) },\quad
  t \geq 1,\ i\in [d]. \]

The following general regret guarantee for strongly convex regularizers
is expressed in terms of the dual norm $\left\| \,\cdot\, \right\|_*$ of $\left\| \,\cdot\,
  \right\|$.
\begin{theorem}[\cite{shalev2011online} Th.\ 2.21; \cite{bubeck2012regret} Th.\ 5.6; \cite{kwon2014continuous} Th.\ 5.1]
\label{thm:mirror-descent}
Let $K>0$ and assume $h$ to be $K$-strongly convex with respect to a norm $\left\|
  \,\cdot\, \right\|$. Then, for all sequence of outcome vectors
$(\omega_t)_{t\geq 1}$ in $\mathbb{R}^d$, the Online Mirror Descent
strategy associated with $h$ and $\eta$ (with $\eta >0$ in cases of gains and $\eta <0$ in cases of losses) guarantees, for $T\geq 1$, the following
regret bound:
\[ R_T\leq \frac{\delta_h}{|\eta|}+\frac{|\eta|}{2K}\sum_{t=1}^T\left\| \omega_t \right\|_*^2.  \]
\end{theorem}

\subsection{Upper Bound on the Regret}
\label{sec:gains}
We first assume $s\geq 2$. Let $p\in (1,2]$ and define the following regularizer:
\[ h_p(x)= \begin{cases}
\frac{1}{2}\left\| x \right\|_p^2&\text{if $x\in \Delta_d$}\\
+\infty&\text{otherwise}.
\end{cases}
 \]
 One can easily check that $h_p$ is indeed a regularizer on $\Delta_d$
 and that $\delta_{h_p}\leqslant 1/2$. Moreover, it is
 $(p-1)$-strongly convex with respect to $\left\|\,\cdot\,
 \right\|_p$ (see \cite[Lemma
5.7]{bubeck2011introduction} or \cite[Lemma 9]{kakade2012regularization}).

% \begin{algorithm}
%   \SetAlgoLined
% \SetKwInOut{Input}{input}
% \Input{$T\geqslant 1$, $d\geqslant 2$ and $2\leqslant s\leqslant d$ integers.}
% $p\leftarrow 1+(2\log s-1)^{-1}$\;
% $q\leftarrow (1-1/p)^{-1}$\;
% $\eta\leftarrow\sqrt{(p-1)/Ts^{2/q}}$\;
% $y\leftarrow (0,\dots,0)\in \mathbb{R}^d$\;
% \For{$t\leftarrow 1$ \KwTo $T$}{
% draw and play decision $i\sim \nabla h^*(\eta\cdot y)$\;
% observe gain vector $g_t$\;
% $y\leftarrow y+g_t$\;
% }
% \caption{For $s$-sparse gains in full information}
% \end{algorithm}

We can now state our first result, the general upper bound on regret when outcomes are $s$-sparse gains.
\begin{theorem}
\label{thm:full-info-gains-m-2}
Let $\eta>0$ and $s\geq 3$. Against all sequence of
$s$-sparse gain vectors $g_t$, i.e.,  $g_t \in [0,1]^d$ and $\|g_t\|_0\leq s$, the Online Mirror Descent algorithm associated
with regularizer $h_p$ and parameter $\eta$ guarantees:
\[ R_T\leq \frac{1}{2\eta}+\frac{\eta Ts^{2/q}}{2(p-1)}, \]
where $1/p+1/q=1$. In particular, the choices $\eta=\sqrt{(p-1)/Ts^{2/q}}$ and  $p=1+(2\log s-1)^{-1}$ give:
\[ R_T\leq \sqrt{2eT\log s}.  \]
\end{theorem}
\begin{proof}
$h_p$ being $(p-1)$-strongly convex with respect to
$\left\| \,\cdot\, \right\|_p$, and $\left\| \,\cdot\, \right\|_q$
being the dual norm of $\left\| \,\cdot\, \right\|_p$, Theorem~\ref{thm:mirror-descent} gives:
\[ R_T\leq \frac{\delta_{h_p}}{\eta}+\frac{\eta}{2(p-1)}\sum_{t=1}^T\left\| g_t \right\|_q^2.     \]
For each $t\geq 1$, the  norm of $g_t$ can be bounded as follows:
\[ \left\| g_t \right\|_q^2=\left( \sum_{i=1}^d\left| g_t^{(i)}\right|^q \right)^{2/q}\leq \left( \sum_{\text{$s$ terms}}^{}\left| g_t^{(i)} \right|^q \right)^{2/q}\leq s^{2/q},   \]
which yields
\[ R_T\leq \frac{1}{2\eta}+\frac{\eta Ts^{2/q}}{2(p-1)}. \]
We can now balance both terms by choosing $\eta=\sqrt{(p-1)/(Ts^{2/q})}$ and get:
\[ R_T\leq \sqrt{\frac{Ts^{2/q}}{p-1}}. \]
Finally, since $s\geq 3$, we have $2\log s>1$ and we  set $p=1+(2\log s-1)^{-1}\in(1,2]$, which gives:
\[ \frac{1}{q}=1-\frac{1}{p}=\frac{p-1}{p}=\frac{(2\log s-1)^{-1}}{1+(2\log s-1)^{-1}}=\frac{1}{2\log s}, \]
and thus:
\[ R_T\leq \sqrt{\frac{Ts^{2/q}}{p-1}}=\sqrt{2T\log s\,e^{2\log s/q}}=\sqrt{2e\, T\log s}. \]
\end{proof}

We emphasize the fact that we obtain, up to a multiplicative constant,
the exact same rate as when the decision maker only has a set of $s$ decisions.

In the case $s=1,2$, we can easily derive a bound of respectively $\sqrt{T}$ and $\sqrt{2T}$
using the same  regularizer with $p=2$.

\subsection{Matching Lower Bound}
\label{sec:fullinfo-gains-lowerbound}
For $s \in [d]$ and $T\geq 1$, we denote $v_T^{g,s,d}$
 the minimax regret of the $T$-stage decision problem with
outcome vectors restricted to $s$-sparse gains:
\[ v_T^{g,s,d}=\min_{\text{strat.}}\max_{ (g_t)_t}  R_T \]
where the minimum is taken over all possible policies of the decision maker,
and the maximum over all sequences of $s$-sparse gains vectors.

To establish a lower bound in the present setting, we can assume that only the first $s$
coordinates of $g_t$ might be positive (for all $t\geqslant 1$) and even that the decision
maker is aware of that. Therefore he has no interest in assigning
positive probabilities to any decision but the first $s$ ones. That
setup, which is simpler for the decision maker than the original one,
is obviously equivalent to the basic regret minimization problem with
only $s$ decisions. Therefore, the classical lower bound
\cite[Theorem 3.2.3]{cesa1997use} holds and we obtain the following.

\begin{theorem}
\label{thm:lowerbound-full-info-gains}
\[ \liminf_{\substack{s\to +\infty\\d\geq s}}\liminf_{T\to +\infty}\frac{v_T^{g,s,d}}{\sqrt{T \log s}}\geq \frac{\sqrt{2}}{2}. \]
\end{theorem}

The same lower bound, up to the multiplicative constant actually holds non asymptotically, see \cite[Theorem~3.6]{cesa2006prediction}.

\medskip
 
An immediate consequence of Theorem \ref{thm:lowerbound-full-info-gains} is that the regret bound
derived in Theorem~\ref{thm:full-info-gains-m-2} is asymptotically minimax
optimal, up to a multiplicative constant.

\section{When Outcomes are Losses to be Minimized}
\label{sec:losses}
\subsection{Upper Bound on the Regret}
%Interestingly, we had to resort to different algorithm to take advantage of the sparsity assumption in the case of losses.
We now consider the case of losses, and the regularizer shall no longer depend on $s$ (as with gains), as we will always use the Exponential Weight Algorithm. Instead, it is  the parameter $\eta$ that will be tuned as a function of $s$.

% \begin{algorithm}
%   \SetAlgoLined
% \SetKwInOut{Input}{input}
% \Input{$T\geqslant 1$, $d\geqslant 1$ and $1\leqslant s\leqslant d$ integers.}
% $\eta\leftarrow \log (1+\sqrt{2d\log d/sT})$\;
% \For{$i\leftarrow 1$ \KwTo $d$}{
%   $w^{(i)}\leftarrow 1/d$\;}
% \For{$t\leftarrow 1$ \KwTo $T$}{
% draw and play decision $i$ with probability
% $w^{(i)}/\sum_{j=1}^{d}w^{(j)}$\;
% observe loss vector $\ell_t$\;
% \For{$i\leftarrow 1$ \KwTo $d$}{
%   $w^{(i)}\leftarrow w^{(i)}e^{-\eta\ell_t^{(i)}}$\;}
% }
% \caption{For $s$-sparse losses in full information}
% \end{algorithm}

\begin{theorem}
\label{thm:upper bound-losses-full-info}
Let $s\geq 1$. For all sequence of $s$-sparse loss vectors $(\ell_t)_{t\geq 1}$, i.e., $\ell_t \in [0,1]^d$ and $\|\ell_t\|_0\leq s$, the Exponential Weight Algorithm with parameter
$-\eta$ where $\eta:=\log \left( 1+\sqrt{2d\log d/sT} \right)>0$ guarantees, for $T\geq 1$:
\[ R_T\leq \sqrt{\frac{2sT\log d}{d}}+\log d. \]
\end{theorem}

We build upon the following regret bound for losses
 which is written in terms of the performance of the best action.
\begin{theorem}[\cite{littlestone1994weighted}; \cite{cesa2006prediction} Th\ 2.4]
\label{thm:upper-bound-small-losses}
Let $\eta>0$. For all sequence of loss vectors $(\ell_t)_{t\geq 1}$ in $[0,1]^d$, the Exponential Weight Algorithm with parameter $-\eta$ guarantees, for all $T\geq 1$:
\[ R_T\leq \frac{\log d}{1-e^{-\eta}}+\left( \frac{\eta}{1-e^{-\eta}}-1 \right)L_T^*\enspace,   \]
where $\displaystyle L_T^*=\min_{i \in [d]}\sum_{t=1}^T\ell_t^{(i)}$ is the loss of the best stationary decision.
\end{theorem}
\begin{proof}[Proof of Theorem~\ref{thm:upper bound-losses-full-info}]
 Let $T\geq 1$ and  $ L_T^*=\min_{i \in [d]}\sum_{t=1}^T\ell_t^{(i)}$
 be the loss of the best stationary policy.
First note that since the loss vectors $\ell_t$ are $s$-sparse, we have  $s\geq \sum_{i=1}^d \ell_t^{(i)}$. By summing over $1\leqslant t\leqslant T$:
\[ sT\geq
\sum_{t=1}^T\sum_{i=1}^d\ell_t^{(i)}=\sum_{i=1}^d\left(  \sum_{t=1}^{T}\ell_t^{(i)}\right) \geq
d\left(  \min_{i \in [d]}\sum_{t=1}^T \ell_t^{(i)}\right)  =dL_T^*, \]
and therefore, we have $L_T^*\leq Ts/d$.

Then, by using the inequality $\eta\leq (e^\eta-e^{-\eta})/2$, the bound from Theorem~\ref{thm:upper-bound-small-losses} becomes:
\[ R_T\leq \frac{\log d}{1-e^{-\eta}}+\left( \frac{e^\eta-e^{-\eta}}{2(1-e^{-\eta})}-1 \right)L_T^*\enspace.  \]
The factor of $L_T^*$ in the second term can be transformed as follows:
\[ \frac{e^\eta-e^{-\eta}}{2(1-e^{-\eta})}-1=\frac{(1+e^{-\eta})(e^{\eta}-e^{-\eta})}{2(1-e^{-2\eta})}-1=\frac{(1+e^{-\eta})e^\eta}{2}-1=\frac{e^\eta-1}{2}\enspace, \]
and therefore the bound on the regret becomes:
\[ R_T\leq \frac{\log d}{1-e^{-\eta}}+\frac{e^\eta-1}{2}L_T^*\leq \frac{\log d}{1-e^{-\eta}}+\frac{(e^{\eta}-1)Ts}{2d}\enspace, \]
where we have been able to use the upper-bound on $L_T^*$ since $\frac{e^\eta-1}{2}\geq 0$. Along with the choice $\eta=\log (1+\sqrt{2d\log d/Ts} )$ and standard computations, this yields:\[ R_T\leq \sqrt{\frac{2Ts\log d}{d}}+\log d\enspace. \]
\end{proof}

Interestingly, the bound from Theorem \ref{thm:upper
  bound-losses-full-info} shows that $\sqrt{2sT\log d/d}$, the
dominating term of the regret bound, is \emph{decreasing} when the
number of decisions $d$ increases. This is due to the sparsity
assumptions (as the regret increases with $s$, the maximal number of
decision with positive losses). Indeed, when $s$ is fixed and $d$
increases, more and more decisions are optimal at each stage, a
proportion $1-s/d$ to be precise. As a consequence, it becomes
\emph{easier} to find an optimal decisions when $d$ increases. However,
this intuition will turn out not to be valid in the bandit framework.

On the other hand, if the proportion $s/d$ of positive losses remains constant then the regret bound achieved is of the same order as in the usual case.

\subsection{Matching Lower Bound}
\label{sec:losses-1}

When outcomes are losses, the argument from
Section~\ref{sec:fullinfo-gains-lowerbound} does not allow to derive a
lower bound. Indeed, if we assume that only the first $s$ coordinates
of the loss vectors $\ell_t$ can be positive, and that the decision
maker knows it, then he just has to take at each stage the decision
$d_t=d$ which incurs a loss of 0. As a consequence, he trivially has a
regret $R_T=0$. Choosing at random, but once and for all, a fixed
subset of $s$ coordinates does not provide any interesting lower bound
either. Instead, the key idea of the following result is to choose at
random and at each stage the $s$ coordinates associated to positive losses.
And we therefore use the following classical probabilistic
argument. Assume that we have found a probability distribution on
$(\ell_t)_t$ such that the expected regret can be bounded from below
by a quantity which does not depend on the strategy of the decision
maker. This would imply that for any algorithm, there exists a sequence of $(\ell_t)_t$
such that the regret is greater than the same quantity.
 
In the following statement, $v_T^{\ell,s,d}$ stands for the minimax
regret in the case where outcomes are losses.
 
\begin{theorem}
  \label{thm:lower bound-losses}
For all $s\geq 1$,
\[ \liminf_{d\to +\infty}\liminf_{T \to +\infty}\frac{v_T^{\ell,s,d}}{\sqrt{T\frac{s}{d}\log d}}\geq \frac{\sqrt{2}}{2}. \]
\end{theorem}
The main consequences of this theorem are that the algorithm described
in Theorem~\ref{thm:upper bound-losses-full-info} is  asymptotically minimax
optimal (up to a multiplicative constant) and that gains and losses are fundamentally different from the point of view of regret minimization.

\begin{proof}
We first define the sequence of loss vectors  $\ell_t$ ($t\geq 1$) i.i.d. as follows. Firs, we draw a set $I_t\subset [d]$ of cardinal $s$ uniformly among the $\binom{d}{s}$ possibilities. Then, if $i\in I_t$ set $\ell_t^{(i)}=1$ with probability $1/2$ and $\ell_t^{(i)}=0$ with probability $1/2$, independently for each component. If $i\not \in I_t$, we set $\ell_t^{(i)}=0$. 

As a consequence, we always have that $\ell_t$ is $s$-sparse. Moreover, for each $t\geq 1$ and each coordinate $i\in [d]$, $\ell_t^{(i)}$ satisfies:
% \[ U_k^i= \begin{cases}
% -1&\text{with probability $s/2d$}\\
% 0&\text{with probability $1-s/2d$},
% \end{cases} \]
\[ \mathbb{P}\left[ \ell_t^{(i)}=1 \right]=\frac{s}{2d}\quad \text{and}\quad \mathbb{P}\left[ \ell_t^{(i)}=0 \right]=1-\frac{s}{2d}\enspace,  \]
thus  $\mathbb{E}\left[ \ell_t^{(i)} \right]=s/2d$.
Therefore we obtain that for any algorithm $(x_t)_{t\geq 1}$,
$\mathbb{E}\left[ \langle \ell_t , x_t \rangle
\right]=s/2d$. This yields that 
\begin{align*}
\mathbb{E}\left[ \frac{R_T}{\sqrt{T}} \right]&=\mathbb{E}\left[ \frac{1}{\sqrt{T}}\left( \sum_{t=1}^T \langle \ell_t , x_t \rangle -\min_{i \in [d]}\sum_{t=1}^T\ell_t^{(i)}\right)  \right]  \\
&=\mathbb{E}\left[ \max_{i \in [d]}\frac{1}{\sqrt{T}}\sum_{t=1}^T\left(\frac{s}{2d}- \ell_t^{(i)} \right)  \right]\\
&=\mathbb{E}\left[ \max_{i \in [d]}\frac{1}{\sqrt{T}}\sum_{t=1}^TX_t^{(i)} \right], 
\end{align*}
where $t\geq 1$, we have defined the random vector $X_t$ by $X_t^{(i)}=s/2d-\ell_t^{(i)}$ for all $i\in [d]$. For $t\geq 1$, the $X_t$ are i.i.d. zero-mean random vectors with values in $[-1,1]^d$. We can therefore apply the comparison Lemma~\ref{lm:1} to get:
\[ \liminf_{T\to +\infty}\mathbb{E}\left[ \frac{R_T}{\sqrt{T}} \right]=\liminf_{T\to +\infty}\mathbb{E}\left[ \max_{i \in [d]}\frac{1}{\sqrt{T}}\sum_{t=1}^TX_t^{(i)} \right]\geq \mathbb{E}\left[ \max_{i \in [d]}Z^{(i)} \right]\enspace,    \]
where $Z \sim \mathcal N(0,\Sigma)$ with $\Sigma=(\operatorname{cov}(X_1^{(i)},X_1^{(j)}))_{i,j}$.

We now make appeal to Slepian's lemma, recalled in
Proposition~\ref{prop:slepian} below. Therefore, we introduce the
Gaussian vector $W \sim \mathcal N( 0, \tilde{\Sigma})$ where
\[ \tilde{\Sigma}=\operatorname{diag}\left(
\operatorname{Var}X_1^{(1)},\dots,\operatorname{Var}X_1^{(1)}
\right). \]
As a consequence, the first two hypotheses of Proposition~\ref{prop:slepian} from the definitions of $Z$ and $W$. Let $i\neq j$, then
\begin{align*}
\mathbb{E}\left[ Z^{(i)}Z^{(j)} \right]&=\operatorname{cov}(Z^{(i)},Z^{(j)})=\operatorname{cov}(\ell_1^{(i)},\ell_1^{(j)})=\mathbb{E}\left[ \ell_1^{(i)}\ell_1^{(j)} \right]-\mathbb{E}\left[ \ell_1^{(i)} \right]\mathbb{E}\left[ \ell_1^{(j)} \right].
\end{align*}
By definition of $\ell_1$, $\ell_1^{(i)}\ell_1^{(j)}=1$ if and only if $\ell_1^{(i)}=\ell_1^{(j)}=1$ and $\ell_1^{(i)}\ell_1^{(j)}=0$ otherwise. Therefore, using the random subset $I_1$ that appears in the definition of $\ell_1$:
\begin{align*}
\mathbb{E}\left[ Z^{(i)}Z^{(j)} \right]&=\mathbb{P}\left[ \ell_1^{(i)}=\ell_1^{(j)}=1 \right]-\left( \frac{s}{2d} \right)^2 \\
&=\mathbb{P}\left[ \ell_1^{(i)}=\ell_1^{(j)}=1\,\middle|\, \left\{ i,j \right\}\subset I_1 \right] \mathbb{P}\left[ \left\{ i,j \right\}\subset I_1  \right]-\left( \frac{s}{2d} \right)^2 \\
&=\frac{1}{4}\cdot \frac{\binom{d-2}{s-2}}{\binom{d}{s}}-\left( \frac{s}{2d} \right)^2\\
&=\frac{1}{4}\left( \frac{s(s-1)}{d(d-1)}-\frac{s^2}{d^2} \right)\leq 0,\\
\end{align*}
and since $\mathbb{E}\left[ W^{(i)}W^{(i)} \right]=0$, the third hypothesis of Slepian's lemma is also satisfied.  It yields that,  for all $\theta\in \mathbb{R}$:
\begin{align*}
\mathbb{P}\left[ \max_{i \in [d]}Z^{(i)}\leq \theta \right]&=\mathbb{P}\left[ Z^{(1)}\leq \theta,\dots, Z^{(d)}\leq \theta\right]\\
&\leq \mathbb{P}\left[ W^{(1)}\leq \theta,\dots, W^{(d)} \leq \theta \right]=\mathbb{P}\left[ \max_{i \in [d]}W^{(i)}\leq \theta \right].     
\end{align*}
This inequality between two cumulative distribution functions implies, the reverse inequality on expectations:
\[ \mathbb{E}\left[ \max_{i \in [d]}Z^{(i)} \right]\geq \mathbb{E}\left[ \max_{i \in [d]}W^{(i)}\right]\enspace.   \]
The components of the Gaussian vector $W$ being independent, and of variance $\operatorname{Var}\ell_1^{(1)}$, we have
\[ \mathbb{E}\left[ \max_{i \in [d]} W^{(i)} \right]= \kappa_d\sqrt{\operatorname{Var}\ell_1^{(1)}}=\kappa_d\sqrt{\frac{s}{2d}\left( 1-\frac{s}{2d} \right) }\geq \kappa_d\sqrt{\frac{s}{4d}}\enspace, \]
where $\kappa_d$ is the expectation of the maximum of $d$ Gaussian variables. Combining everything gives:
\[ \liminf_{T\to +\infty}\frac{v_T^{\ell,s,d}}{\sqrt{T}}\geq \liminf_{T\to +\infty}\mathbb{E}\left[ \frac{R_T}{\sqrt{T}} \right]\geq \mathbb{E}\left[ \max_{i \in [d]} Z^{(i)} \right]\geq \mathbb{E}\left[ \max_{i \in [d]} W^{(i)} \right]\geq \kappa_d\sqrt{\frac{s}{4d}}\enspace.    \]
And for large $d$, since $\kappa_d$ is equivalent to $\sqrt{2\log d}$, see e.g., \cite{galambos1980asymptotic}
\[ \liminf_{d\to +\infty}\liminf_{T\to +\infty}\frac{v_T^{\ell,s,d}}{\sqrt{T\frac{s}{d}\log d}}\geq \frac{\sqrt{2}}{2}\enspace. \]
\end{proof}

\begin{proposition}[Slepian's lemma \cite{slepian1962one}]
\label{prop:slepian}
Let $Z=(Z^{(1)},\dots,Z^{(d)})$ and $W=(W^{(1)},\dots,W^{(d)})$ be Gaussian random vectors in $\mathbb{R}^d$ satisfying:
\begin{enumerate}[(i)]
\item $\mathbb{E}\left[ Z \right]=\mathbb{E}\left[ W \right]=0$;
\item $\mathbb{E}\left[ (Z^{(i)})^2 \right]=\mathbb{E}\left[ (W^{(i)})^2 \right]$ for $i \in [d]$;
\item $\mathbb{E}\left[ Z^{(i)}Z^{(j)} \right]\leq \mathbb{E}\left[W^{(i)}W^{(j)}\right]$ for $i\neq j \in [d]$.
\end{enumerate}
Then, for all real numbers $\theta_1,\dots,\theta_d$, we have:
\[ \mathbb{P}\left[ Z^{(1)}\leq \theta_1,\dots,Z^{(d)}\leq \theta_d \right]\leq \mathbb{P}\left[ W^{(1)}\leq \theta_1,\dots,W^{(d)}\leq \theta_d \right]\enspace.   \]
\end{proposition}

The following lemma is an extension of e.g.\ \cite[Lemma
A.11]{cesa2006prediction} to random vectors with correlated components.
\begin{lemma}[Comparison lemma]
\label{lm:1}
For $t \geq 1$, let $(X_t)_{t\geqslant 1}$ be i.i.d. zero-mean random vectors in $[-1,1]^d$, $\Sigma$ be the covariance matrix of $X_t$ and $Z \sim \mathcal N(0,\Sigma)$. Then,
\[ \liminf_{T\to +\infty}\mathbb{E}\left[ \max_{i \in [d]}\frac{1}{\sqrt{T}}\sum_{t=1}^TX_t^{(i)} \right]\geq  \mathbb{E}\left[ \max_{i \in [d]} Z^{(i)} \right].    \]
\end{lemma}
\begin{proof}
Denote
\[ Y_T=\max_{i\in [d]}\frac{1}{\sqrt{T}}\sum_{t=1}^TX_t^{(i)}. \]
Let $A\leqslant 0$ and consider the function $\phi_A:\mathbb{R}\to \mathbb{R}$ defined by $\phi_A(x)=\max_{}(x,A)$. 
\begin{align*}
\mathbb{E}\left[ Y_T \right]&=\mathbb{E}\left[ Y_T\cdot \mathbbm{1}_{\left\{ Y_T\geqslant A \right\} } \right]+\mathbb{E}\left[ Y_T\cdot \mathbbm{1}_{\left\{ Y_T<A \right\} } \right]\\
&=\mathbb{E}\left[ \phi_A(Y_T)\cdot \mathbbm{1}_{\left\{ Y_T\geqslant A \right\} } \right]+\mathbb{E}\left[ Y_T\cdot \mathbbm{1}_{\left\{ Y_T<A \right\} } \right]\\
&=\mathbb{E}\left[ \phi_A(Y_T) \right]-\mathbb{E}\left[ \phi_A(Y_T)\cdot \mathbbm{1}_{\left\{ Y_T<A \right\} } \right]+\mathbb{E}\left[ Y_T\cdot \mathbbm{1}_{\left\{ Y_T<A \right\} } \right]\\
&=\mathbb{E}\left[ \phi_A(Y_T) \right]-\mathbb{E}\left[ (A-Y_T)\cdot \mathbbm{1}_{\left\{ A-Y_T>0 \right\} } \right].  
\end{align*}

Let us estimate the second term. Denote $Z_T=(A-Y_T)\cdot \mathbbm{1}_{A-Y_T>0}$. We clearly have, for all $u>0$, $\mathbb{P}\left[ Z_T>u \right]=\mathbb{P}\left[ A-Y_T>u \right]$. And $Z_T$ being nonnegative, we can write:
\begin{align*}
0&\leqslant \mathbb{E}\left[ (A-Y_T)\cdot \mathbbm{1}_{\left\{ A-Y_T \right\}>0 } \right]=\mathbb{E}\left[ Z_T \right]\\
&=\int_0^{+\infty}\mathbb{P}\left[ Z_T>u \right]\,\mathrm{d}u\\
&=\int_0^{+\infty}\mathbb{P}\left[ A-Y_T>u \right]\,\mathrm{d}u \\
&=\int_{-A}^{+\infty}\mathbb{P}\left[ Y_T<-u \right]\,\mathrm{d}u\\
&=\int_{-A}^{+\infty}\mathbb{P}\left[ \max_{i\in [d]}\frac{1}{\sqrt{T}}\sum_{t=1}^TX_t^{(i)}<u \right]\,\mathrm{d}u\\
&\leqslant \int_{-A}^{+\infty}\mathbb{P}\left[ \sum_{t=1}^TX_t^{(1)}<u\sqrt{T} \right]\,\mathrm{d}u. 
\end{align*}
For $u>0$, using Hoeffding's inequality together with the assumptions  $\mathbb{E}\left[ X_t^{(1)} \right]=0$ and $X_t^{(1)}\in [-1,1]$, we can bound the last integrand: 
\[ \mathbb{P}\left[ \sum_{t=1}^TX_t^{(1)}<u\sqrt{T} \right]\leqslant  e^{-u^2/2},   \]
Which gives:
\[ 0\leqslant \mathbb{E}\left[ (A-Y_T)\cdot \mathbbm{1}_{\left\{ A-Y_T \right\}>0 } \right]\leqslant \int_{-A}^{+\infty}e^{-u^2/2}\,\mathrm{d}u\leqslant \frac{e^{-A^2/2}}{-A}. \]
Therefore:
\[ \mathbb{E}\left[ Y_T \right]\geqslant \mathbb{E}\left[ \phi_A(Y_T) \right]+\frac{e^{-A^2/2}}{A}.   \]
We now take the liminf on both sides as $t\to +\infty$. The left-hand side is the quantity that appears in the statement. We now focus on the second term of the right-hand side.
The central limit theorem gives the following convergence in distribution:
\[ \frac{1}{\sqrt{T}}\sum_{t=1}^TX_t \xrightarrow[T \rightarrow +\infty]{\mathcal L}X. \]
The application $(x^{(1)},\dots,x^{(d)})\longmapsto \max_{i\in [d]}x^{(i)}$ being continuous, we can apply the continuous mapping theorem:
\[ Y_T=\max_{i\in [d]}\frac{1}{\sqrt{T}}\sum_{t=1}^TX_t^{(i)} \xrightarrow[n \rightarrow +\infty]{\mathcal L}\max_{i\in [d]}X^{(i)}. \]
This convergence in distribution allows the use of the portmanteau lemma: $\phi_A$ being lower semi-continuous and bounded from below, we have:
\[ \liminf_{t\to +\infty}\mathbb{E}\left[ \phi_A(Y_T) \right]\geqslant \mathbb{E}\left[ \phi_A\left( \max_{i\in [d]}X^{(i)} \right)  \right],   \]
and thus:
\[ \liminf_{t\to +\infty}\mathbb{E}\left[ Y_T \right]\geqslant  \mathbb{E}\left[ \phi_A\left( \max_{i\in [d]}X^{(i)}\right)  \right]+\frac{e^{-A^2/2}}{A}.   \]
We would now like to take the limit as $A\to -\infty$. By definition of $\phi_A$, for $A\leqslant 0$, we have the following domination:
\[ \left| \phi_A\left( \max_{i\in [d]}X^{(i)} \right)  \right| \leqslant \left| \max_{i\in [d]}X^{(i)} \right|\leqslant \max_{i\in [d]}\left| X^{(i)} \right| \leqslant  \sum_{i=1}^d\left| X^{(i)} \right|,   \]
where each $X^{(i)}$ is $L^1$ since it is a normal random variable. We can therefore apply the dominated convergence theorem as $A\to -\infty$:
\[ \mathbb{E}\left[ \phi_A\left( \max_{i\in [d]}X^{(i)}  \right)  \right] \xrightarrow[A \rightarrow -\infty]{}\mathbb{E}\left[ \max_{i\in [d]}X^{(i)} \right],    \]
and eventually, we get the stated result:
\[ \liminf_{t\to +\infty}\mathbb{E}\left[ Y_T \right]\geqslant \mathbb{E}\left[ \max_{i\in [d]}X^{(i)} \right].   \]
\end{proof}

\section{When the sparsity level $s$ is unknown}
\label{sec:when-sparsity-level}
We now longer assume in this section that the decision maker have the
knowledge of the sparsity level $s$. We modify our algorithms to be 
adaptive over the sparsity level of the observed gain/loss vectors, following the same ideas behind the classical doubling trick (yet it cannot be directly applied here). The
algorithms are proved to essentially achieve the same regret bounds as in the case
where $s$ is known.

Specifically, let $T\geqslant 1$ be the number of rounds and $s^*$ the
highest sparsity level of the gain/loss vectors chosen by Nature up to time $T$. In the following, we
construct algorithms which achieve regret bounds of order $\sqrt{T\log
  s^*}$ and $\sqrt{T\frac{s^*\log d}{d}}$ for gains and losses
respectively, without prior knowledge of $s^*$.

\subsection{For Losses}
\label{sec:doubling-trick-like}

Let $(\ell_t)_{t\geqslant 1}$ be the sequence of loss vectors in $[0,1]^d$ chosen
by Nature, and $T\geqslant 1$ the number of rounds. We denote
$s^*=\max_{1\leqslant t\leqslant T}\left\| \ell_t \right\|_0$ the higher
sparsity level of the loss vectors up to time $T$. The goal is to
construct an algorithm which achieves a regret bound of order
$\sqrt{\frac{Ts^*\log d}{d}}$ without any prior knowledge about the
sparsity level of the loss vectors.

The time instances
$\left\{ 1,\dots,T \right\}$ will be divided into several time
intervals. On each of those, the previous loss vectors will be left aside,
and a new instance of the Exponential Weight Algorithm with a specific
parameter will be run. Let $M=\lceil \log_2s^* \rceil$ and
$\tau(0)=0$. Then, for $1\leqslant m<M$ we define
\[ \tau(m)=\min_{}\left\{ 1\leqslant t\leqslant T\,\middle|\, \left\| \ell_t \right\|_0>2^m \right\}\quad \text{and}\quad \tau(M)=T. \]
In other words, $\tau(m)$ is the first time instance at which the
sparsity level of the loss vector execeeds
$2^m$. $(\tau(m))_{1\leqslant m\leqslant M}$ is thus a nondecreasing
sequence. We can then define the time intervals $I(m)$ as follows. For $1\leqslant m\leqslant M$, let
\[ I(m)= \begin{cases}
 \left\{ \tau(m-1)+1,\dots,\tau(m) \right\}&\text{if }
 \tau(m-1)<\tau(m)\\
\varnothing&\text{if }\tau(m-1)=\tau(m).
  \end{cases} .  \]
The sets $(I(m))_{1\leqslant m\leqslant M}$ clearly is a partition
of $\left\{ 1,\dots,T \right\}$ (some of the intervals may be empty).
For $1\leqslant t\leqslant T$, we define $m_t=\min_{}\left\{
  m\geqslant 1\,\middle|\,\tau(m)\geqslant t \right\}$ which implies
$t\in I(m_t)$. In other words, $m_t$ is the index of the only interval $t$
belongs to.

Let $C>0$ be a constant to be chosen later and for $1\leqslant
m\leqslant M$, let
\[ \eta(m)=\log \left( 1+C\sqrt{\frac{d\log d}{2^mT}} \right)  \]
  be the parameter of the Exponential Weight Algorithm to
be used on interval $I(m)$.
In this section, $h$ will be entropic regularizer on the simplex
$h(x)=\sum_{i=1}^dx^{(i)}\log x^{(i)}$, so that $y\longmapsto \nabla
h^*(y)$ is the \emph{logit map} used in the Exponential
Weight Algorithm. We can then define the played actions to be:
\[ x_t=\nabla h^*\left( -\eta(m_t)\sum_{\substack{t'<t\\t'\in I(m_t)}}^{}\ell_{t'} \right),\quad t=1,\dots,T.  \]

\begin{algorithm}
  \SetAlgoLined
\SetKwInOut{Input}{input}
\Input{$T\geqslant 1$, $d\geqslant 1$ integers, and $C>0$.}
$\eta\leftarrow \log (1+C\sqrt{d\log d/2T})$\;
$m\leftarrow 1$\;
\For{$i\leftarrow 1$ \KwTo $d$}{
  $w^{(i)}\leftarrow 1/d$\;}
\For{$t\leftarrow 1$ \KwTo $T$}{
draw and play decision $i$ with probability
$w^{(i)}/\sum_{j=1}^{d}w^{(j)}$\;
observe loss vector $\ell_t$\;
\eIf{$\left\| \ell_t \right\|_0\leqslant 2^m$}{
\For{$i\leftarrow 1$ \KwTo $d$}{
  $w^{(i)}\leftarrow w^{(i)}e^{-\eta\ell_t^{(i)}}$\;}
}{
$m\leftarrow \lceil \log_2\left\| \ell_t \right\|_0 \rceil$\;
$\eta\leftarrow \log (1+C\sqrt{d\log d/2^mT})$\;
\For{$i\leftarrow 1$ \KwTo $d$}{
  $w^{(i)}\leftarrow 1/d$\;}
}
}
\caption{For losses in full information without prior knowledge about
sparsity}
\end{algorithm}

\begin{theorem}
\label{thm:adative-losses}
The above algorithm with $C=2^{3/4}(\sqrt{2}+1)^{1/2}$ guarantees
\[ R_T\leqslant 4\sqrt{\frac{Ts^*\log d}{d}}+\frac{\lceil \log s^* \rceil\log d}{2}+5s^*\sqrt{\frac{\log d}{dT}}. \]
\end{theorem}
\begin{proof}
Let $1\leqslant m\leqslant M$. On time interval $I(m)$, the
Exponential Weight Algorithm is run with parameter $\eta(m)$ against
loss vectors in $[0,1]^d$. Therefore, the following regret bound
derived in the proof of Theorem~\ref{thm:upper bound-losses-full-info}
applies:
\begin{align*}
R(m):=&\sum_{t\in I(m)}^{}\left< \ell_t , x_t \right> -\min_{i\in
  [d]}\sum_{t\in I(m)}^{}\ell_t^{(i)}\\
&\leqslant \frac{\log d}{1-e^{-\eta(m)}}+\frac{e^{\eta(m)}-1}{2}\min_{i\in
                                       [d]}\sum_{t\in
                                       I(m)}^{}\ell_t^{(i)}\\
&=\frac{1}{C}\sqrt{\frac{2^mT\log d}{d}}+\frac{\log d}{C}+\frac{C}{2}\sqrt{\frac{d\log d}{2^mT}}\cdot \min_{i\in
                                       [d]}\sum_{t\in
                                       I(m)}^{}\ell_t^{(i)}.
\end{align*}
We now bound the ``best loss'' quantity from above, using the fact
 that $\ell_t$ is $2^m$-sparse for $t\in
I(m)\setminus \left\{ \tau(m) \right\}$ and that $\ell_{\tau(m)}$ is $s^*$-sparse:
\begin{align*}
  \sum_{i=1}^d \sum_{t\in I(m)}^{}\ell_t^{(i)}&=\sum_{t\in I(m)}^{}
                                                \sum_{i=1}^d\ell_t^{(i)} =\sum_{\substack{t<\tau(m)\\t\in I(m)}}^{} \sum_{i=1}^d\ell_t^{(i)}
  +\sum_{i=1}^d \ell_{\tau(m)}^{(i)}\\
&\leqslant (\tau(m)-\tau(m-1))2^m+s^*,
\end{align*}
which implies:
\[ \min_{i\in [d]}\sum_{t\in I(m)}^{}\ell_t^{(i)}\leqslant \frac{(\tau(m)-\tau(m-1))2^m+s^*}{d}. \]
Therefore, the regret on interval $I(m)$, which we will denote $R(m)$,
is bounded by: 
\begin{align*}
R(m)&:=\sum_{t\in I(m)}^{}\left< \ell_t , x_t \right> -\min_{i\in
  [d]}\sum_{t\in I(m)}^{}\ell_t^{(i)}\\
&\leqslant \frac{1}{C}\sqrt{\frac{2^mT\log
                                       d}{d}}+\frac{\log d}{C}+
  \frac{C}{2}\sqrt{\frac{2^m\log d}{dT}}\left(
  \tau(m)-\tau(m-1) \right)+\frac{C}{2}\sqrt{\frac{\log d}{2^mdT}}s^*\\ 
&\leqslant  \frac{1}{C}\sqrt{\frac{2^mT\log
                                       d}{d}}+\frac{\log d}{C}+
  \frac{C}{2}\sqrt{\frac{2s^*\log d}{dT}}\left(
  \tau(m)-\tau(m-1) \right)+\frac{C}{2}\sqrt{\frac{\log d}{2^mdT}}s^*,
\end{align*}
where we used $2^m\leqslant 2^M=2^{\lceil \log_2s^* \rceil}\leqslant
2^{\log_2s^*+1}=2s^*$ for the third term of the last line.

We now turn the whole regret $R_T$ from $1$ to $T$. Since $(I(m))_{1\leqslant
  m\leqslant M}$ is a partition of $\left\{ 1,\dots,T \right\}$, we obtain
\begin{align*}
R_T&=\sum_{t=1}^T\left< \ell_t , x_t \right> -\min_{i\in
     [d]}\sum_{t=1}^T\ell_t^{(i)}\\
&\leqslant \sum_{m=1}^M \sum_{t\in I(m)}^{}\left< \ell_t , x_t \right> -\sum_{m=1}^M\min_{i\in [d]}\sum_{t\in I(m)}^{}\ell_t^{(i)}\\
 &=\sum_{m=1}^MR(m)\\
&\leqslant \frac{1}{C}\sqrt{\frac{T\log d}{d}}\sum_{m=1}^M\sqrt{2^m}+C\sqrt{\frac{s^*T\log d}{2d}}+\frac{M\log d}{C}+\frac{C}{2}\sqrt{\frac{\log d}{dT}}s^*\sum_{m=1}^M2^{-m/2}.
\end{align*}
The sum in the first term above can be bounded as follows
\[ \sum_{m=1}^M\sqrt{2^m}\leqslant \sum_{m=1}^M\sqrt{2}^m=\sqrt{2}\frac{\sqrt{2}^M-1}{\sqrt{2}-1}\leqslant \sqrt{2}\frac{\sqrt{2^{\log_2s^*+1}}}{\sqrt{2}-1}=2\frac{\sqrt{s^*}}{\sqrt{2}-1}=2(\sqrt{2}+1)\sqrt{s^*},  \]
whereas the sum in the last term can be bounded by $\sqrt{2}+1$.
Eventually, the choice $C=2^{3/4}(\sqrt{2}+1)^{1/2}$  give:
\[ R_T\leqslant 2^{5/4}(\sqrt{2}+1)^{1/2}\sqrt{\frac{Ts^*\log
      d}{d}}+\frac{\lceil \log s^* \rceil \log d
  }{2^{3/4}(\sqrt{2}+1)^{1/2}}+2^{1/4}(\sqrt{2}+1)^{3/2}s^*\sqrt{\frac{\log d}{dT}},  \]
and the statement follows from numerical computation of the constant factors.

% \[ R_T\leqslant 2\sqrt{\sqrt{2}+1}\sqrt{\frac{Ts^*\log d}{d}}+\sqrt{\frac{\sqrt{2}-1}{2}}\lceil \log_2s^* \rceil\log d+(\sqrt{2}+1)^{3/2}s^*\sqrt{\frac{\log d}{dT}}.  \]
\end{proof}

\subsection{For Gains}
\label{sec:gains-1}
The construction is similar to the case of losses, but the time
intervals are slightly different.
Let $(g_t)_{t\geqslant 1}$ be the sequence of gain vectors in
$[0,1]^d$ chosen by Nature.
We assume $s^*\geqslant 2$ and set $M=\lceil \log_2\log_2s^* \rceil$
and $\tau(0)=0$. For $1\leqslant m\leqslant M$ we define
\[ \tau(m)=\min_{}\left\{ 1\leqslant t\leqslant T\,\middle|\,\left\| g_t \right\|_0>2^{2^m} \right\}\quad \text{and}\quad \tau(M)=T.  \]
We now define the time intervals $I(m)$. For $1\leqslant m\leqslant M$,
\[ I(m)= \begin{cases}
 \left\{ \tau(m-1)+1,\dots,\tau(m) \right\}&\text{if }
 \tau(m-1)<\tau(m)\\
\varnothing&\text{if }\tau(m-1)=\tau(m).
  \end{cases}  \]
Therefore, for $1\leqslant m\leqslant M$ and $t<\tau(m)$, we have $\left\| g_t \right\|_0\leqslant 2^{2^{m}}$.
For $1\leqslant t\leqslant T$, we denote $m_t=\min_{}\left\{ m\geqslant 1\,\middle|\,\tau(m)\geqslant t \right\}$.
Let $C>0$ be a constant to be chosen later and for $1\leqslant
m\leqslant M$, let
\begin{align*}
p(m)&=1+\frac{1}{\log 2\cdot 2^{m+1}-1},\\
q(m)&=\left(  1-\frac{1}{p(m)}\right)^{-1},\\
\eta(m)&=C\sqrt{\frac{p(m)-1}{T2^{2^{m+1}/q(m)}}}.
\end{align*}
As in Section~\ref{sec:gains}, for $p\in (1,2]$, we denote $h_p$ the regularizer on the simplex defined by:
\[ h_p(x)= \begin{cases}
\frac{1}{2}\left\| x \right\|_p^2&\text{if $x\in \Delta_d$}\\
+\infty&\text{otherwise}.
\end{cases}
 \]
The algorithm is then defined by:
\[ x_t=\nabla h_{p(m_t)}^*\left( \eta(m_t)\sum_{\substack{t'<t\\t'\in I(m_t)}}^{}g_{t'} \right),\quad t=1,\dots,T.  \]
\begin{algorithm}
  \SetAlgoLined
\SetKwInOut{Input}{input}
\Input{$T\geqslant 1$, $d\geqslant 1$ integers, and $C>0$.}
$p\leftarrow 1+(4\log 2-1)^{-1}$\;
$q\leftarrow (1-1/p)^{-1}$\;
$\eta\leftarrow C\sqrt{(p-1)/2^{4/q}T}$\;
$m\leftarrow 1$\;
$y\leftarrow (0,\dots,0)\in \mathbb{R}^d$\;
\For{$t\leftarrow 1$ \KwTo $T$}{
draw and play decision $i\sim \nabla h^*_p(\eta\cdot y)$\; 
observe gain vector $g_t$\;
\eIf{$\left\| g_t \right\|_0\leqslant 2^{2^{m}}$}{
 $y\leftarrow y+g_t$\;
}{
$m\leftarrow \lceil \log_2\log_2\left\| g_t \right\|_0 \rceil$\;
$p\leftarrow 1+(\log 2\cdot 2^{m+1}-1)^{-1}$\;
$q\leftarrow (1-1/p)^{-1}$\;
$\eta\leftarrow C\sqrt{(p-1)/2^{2^{m+1}/q}T}$\;
$y\leftarrow (0,\dots,0)$\;}}
\caption{For gains in full information without prior knowledge about sparsity.}
\end{algorithm}

\begin{theorem}
\label{thm:adaptive-gains}
The above algorithm with $C=(e\sqrt{2}(\sqrt{2}+1))^{1/2}$ guarantees
\[ R_T\leqslant 7\sqrt{T\log s^*}+\frac{4s^*}{\sqrt{T}}. \]
\end{theorem}
\begin{proof}
Let $1\leqslant m\leqslant M$. On time interval $I(m)$, the algorithm boils down to an Online Mirror
Descent algorithm with regularizer $h_{p(m)}$ and parameter
$\eta(m)$. Therefore, using Theorem~\ref{thm:mirror-descent}, the regret on this interval is bounded as follows.
\begin{align*}
R(m)&:=\max_{i\in [d]}\sum_{t\in I(m)}^{}g_t^{(i)}-\sum_{t\in
  I(m)}^{}\left< g_t , x_t \right>\\
&\leqslant \frac{1}{2\eta(m)}+\frac{\eta(m)}{2(p(m)-1)}\sum_{t\in
  I(m)}^{}\left\| g_t \right\|_{q(m)}^2\\
&=\frac{1}{2\eta(m)}+\frac{\eta(m)}{2(p(m)-1)}\left( \sum_{\substack{t\in I(m)\\t<\tau(m)}}^{}\left\| g_t \right\|_{q(m)}^2+\left\| g_{\tau(m)} \right\|_{q(m)}^2 \right).
\end{align*}
$g_t$ being $2^{2^m}$-sparse for $t<\tau(m)$ and $g_{\tau(m)}$ being
$s^*$-sparse, the $q(m)$-norms can therefore bounded from above as
follows:
\[ \left\| g_t \right\|_{q(m)}^2\leqslant 2^{2^{m+1}/q(m)}\quad \text{and}\quad \left\| g_{\tau(m)} \right\|_{q(m)}^2\leqslant (s^*)^{2/q(m)}. \]
The bound on $R(m)$ then becomes
\begin{align*}
R(m)&\leqslant
  \frac{1}{2\eta(m)}+\frac{\eta(m)(\tau(m)-\tau(m-1))2^{2^{m+1}/q(m)}}{2(p(m)-1)}+\frac{\eta(m)(s^*)^{2/q(m)}}{2(p(m)-1)}\\ 
&=\frac{1}{2C}\sqrt{Te(\log 2\cdot
  2^{m+1}-1)}+\frac{C}{2}\sqrt{\frac{e(\log 2\cdot
  2^{m+1}-1)}{T}}(\tau(m)-\tau(m-1))\\
&\quad \quad \quad \quad \quad +\frac{C}{2}(s^*)^{1/(\log 2\cdot 2^m)}\sqrt{\frac{e(\log 2\cdot 2^{m+1}-1)}{T}}\\
&\leqslant \frac{1}{2C}\sqrt{Te\log 2\cdot 2^{m+1}}+C\sqrt{\frac{e\log
  s^*}{T}}\left( \tau(m)-\tau(m-1) \right)\\
&\quad \quad \quad \quad \quad +\frac{C}{2}s^*\sqrt{\frac{e\log 2\cdot 2^{m+1}}{T}},
\end{align*}
where for the second term of the last expression we used:
\begin{align*}
\log 2\cdot 2^{m+1}-1&\leqslant \log 2\cdot 2^{M+1}=\log 2\cdot \exp \left(
                       \log 2\left( \lceil \log_2\log_2s^* \rceil
                       +1\right)  \right) \\
&\leqslant \log 2\cdot \exp \left( \log 2\left( \log_2\log_2s^*+2
  \right)  \right)\\
&=\log 2\cdot e^{2\log 2}\exp \left( \log 2\cdot \log_2\log_2s^*
  \right)\\
&= 4\log 2\cdot \exp \left( \log \log_2s^* \right)\\
&=4\log 2\cdot \log_2s^* \\
&=4\log s^*.
\end{align*}
Then, the whole regret $R_T$ is bounded by the sum of the regrets on
each interval:
\begin{align*}
R_T&\leqslant \sum_{m=1}^MR(m)\leqslant \frac{1}{2C}\sqrt{Te\log
     2}\sum_{m=1}^M\sqrt{2^{m+1}}+C\sqrt{\frac{e\log
     s^*}{T}}\sum_{m=1}^M(\tau(m)-\tau(m-1))\\
&\quad \quad \quad \quad \quad +\frac{Cs^*}{2}\sqrt{\frac{e\log 2}{T}}\sum_{m=1}^M2^{-(m+1)/2}.
\end{align*}
The second sum is equal to $T$ and the third sum is bounded from above
by $(\sqrt{2}+1)/\sqrt{2}$. Let us bound the first sum from above:
\begin{align*}
\sqrt{\log 2}\sum_{m=1}^M\sqrt{2^{m+1}}&=2\sqrt{\log
                                         2}\frac{2^{M/2}-1}{\sqrt{2}-1}\\
&\leqslant 2(\sqrt{2}+1)\sqrt{\log 2}\cdot \exp \left(
                                         \frac{\log 2}{2}\left(
                                         \log_2\log_2s^*+1 \right)
                                         \right) \\
&=2(\sqrt{2}+1)\sqrt{\log 2}\cdot \sqrt{2e^{\log \log_2s^*}}\\
&=2\sqrt{2}(\sqrt{2}+1)\sqrt{\log 2 \log_2s^*}\\
&=2\sqrt{2}(\sqrt{2}+1)\sqrt{\log s^*}.
\end{align*}
 Therefore,
\[ R_T\leqslant \frac{\sqrt{2}(\sqrt{2}+1)}{C}\sqrt{Te\log s^*}+C\sqrt{Te\log s^*}+\frac{C(\sqrt{2}+1)s^*}{2}\sqrt{\frac{e\log 2}{2T}}. \]
Choosing $C=(e\sqrt{2}(\sqrt{2}+1))^{1/2}$ balance the first two term
and gives:
\begin{align*}
R_T&\leqslant 2(e\sqrt{2}(\sqrt{2}+1))^{1/2}\sqrt{T\log s^*}+2^{-5/4}e\sqrt{\log 2}(\sqrt{2}+1)^{3/2}\frac{s^*}{\sqrt{T}}\\
&\leqslant 7\sqrt{T\log s^*}+\frac{4s^*}{\sqrt{T}}.
\end{align*}
\[  \]
\end{proof}

\section{The Bandit Setting}
\label{sec:bandit}
We now turn to  the bandit framework (see for instance \cite{bubeck2012regret}
for a recent survey).
Recall that the minimax regret \cite{audibert2009minimax} in the basic bandit framework
(without sparsity) is of order $\sqrt{Td}$.
In the case of losses, we manage to take advantage of the sparsity
assumption and obtain in Theorem~\ref{thm:bandit-upper bound-losses} an upper bound of order $\sqrt{Ts\log
\frac{d}{s}}$, and an lower bound of order
$\sqrt{Ts}$ in Theorem~\ref{thm:bandit-losses-lowerbound}. This establishes the order of the minimax regret up to a
logarithmic factor. In the case of gains however, the same upper bound
and lower bound techniques do not seem to work; this difficulty is
discussed below in remark \ref{RM:upper_bound_bandit}.

%The $\sqrt{Td}$ upper bound from \cite{audibert2009minimax} is therefore the best that we have at the moment, and the best lower bound is of order $\sqrt{Ts}$ (which is easy in the case of gains). Consequently, the minimax regret in the case of sparse gains  remains an open problem.

For simplicity, we shall assume that the
sequence of outcome vectors $(\omega_t)_{t\geq 1}$ is chosen before stage 1 by the
environment, which is called \emph{oblivious} in that case. We refer to \cite[Section 3]{bubeck2012regret} for a detailed discussion
on the difference between oblivious and non-oblivious opponent, and
between regret and pseudo-regret. 

As before, at stage $t$, the
decision maker chooses $x_t \in
\Delta_d$ and draws decision $d_t \in [d]$ according to $x_t$. The main difference with the previous framework is that the decision maker only observes
his own outcome $\omega_t^{d_t}$ before choosing the next decision $d_{t+1}$. 

\subsection{Upper Bounds on the Regret with Sparse Losses}
\label{sec:greedy-online-mirror}
We shall focus in this section on $s$-sparse losses. The algorithm we consider belongs to the family of Greedy Online Mirror
Descent. We follow \cite[Section 5]{bubeck2012regret} and refer to it
for the detailed and rigorous construction. Let $F_q(x)$ be
the Legendre function associated with potential $\psi(x)=(-x)^{-q}$
($q>1$), i.e.,
\[ F_q(x)=-\frac{q}{q-1}\sum_{i=1}^d(x^i)^{1-1/q}. \]
The algorithm,  which depends on a parameter $\eta>0$ to be fixed later,   is defined as follows. Set
$x_1=(\frac{1}{d},\dots,\frac{1}{d}) \in \Delta_d$. For all $t\geq 1$,
we define the estimator $\hat{\ell}_t$ of $\ell_t$ as usual:
\[ \hat{\ell}_t^{(i)}=\mathbbm{1}_{\left\{  d_t=i\right\} }\frac{\ell_t^{(i)}}{x_t^{(i)}},\quad i \in [d], \]
which is then used to compute
\[ z_{t+1}=\nabla F_q^*(\nabla F_q(x_t)-\eta \hat{\ell}_t)\quad \text{and}\quad x_{t+1}=\operatorname{argmin}_{x\in \Delta_d}D_{F_q}(x,z_{t+1}), \]
where $D_{F_q}: \bar{\mathcal D}\times \mathcal D\to \mathbb{R}$ is the
Bregman divergence associated with $F_q$:
\[ D_{F_q}(x',x)=F_q(x')-F_q(x)-\left< \nabla F_q(x) , x'-x \right>. \]
\begin{theorem}
\label{thm:bandit-upper bound-losses}
Let $\eta>0$ and $q>1$. For all sequence of $s$-sparse loss vectors,
the above strategy with parameter $\eta$ guarantees, for $T\geq
1$:
\[ R_T\leq q\left( \frac{d^{1/q}}{\eta(q-1)}+\frac{\eta Ts^{1-1/q}}{2} \right).  \]
In particular, if $d/s\geq e^2$, the choices
\[ \eta=\sqrt{\frac{2d^{1/q}}{(q-1)Ts^{1-1/q}}}\quad \text{and}\quad
  q=\log (d/s) \]
the following regret bound:
\[ R_T\leq 2\sqrt{e}\sqrt{Ts\log \frac{d}{s}}. \]
\end{theorem}
\begin{proof}
\cite[Theorem 5.10]{bubeck2012regret} gives:
\[ R_T\leq \frac{\max_{x\in \Delta_d}F(x)-F(x_1)}{\eta}+\frac{\eta}{2}\sum_{t=1}^T\sum_{i=1}^{d}\mathbb{E}\left[ \frac{(\hat{\ell}_t^{(i)})^2}{(\psi^{-1})'(x_t^{(i)})} \right], \]
with $(\psi^{-1})'(x)=(q\, x^{1+1/q})^{-1}$.
Let us bound the first term.
\[ \frac{1}{\eta}\max_{x\in \Delta_d}F_q(x)-F_q(x_1)\leq \frac{1}{\eta}\frac{q}{q-1}\left(0+d\left( 1/d \right)^{1-1/q} \right)=\frac{qd^{1/q}}{\eta(q-1)}. \]
We turn to the second term. Let $1\leq t\leq T$.
\begin{align*}
\sum_{i=1}^{d}\mathbb{E}\left[
  \frac{(\hat{\ell}_t^{(i)})^2}{(\psi^{-1})'(x_t^{(i)})}
  \right]&=q\sum_{i=1}^d\mathbb{E}\left[
           (\hat{\ell}_t^{(i)})^2(x_t^{(i)})^{1+1/q}
           \right]\\
&=q\sum_{i=1}^d\mathbb{E}\left[ \mathbb{E}\left[
           \mathbbm{1}_{\left\{ d_t=i \right\}
           }\frac{(\ell_t^{(i)})^2}{(x_t^i)^2}(x_t^i)^{1+1/q}\middle|x_t\right]
           \right]\\
  &=q\sum_{i=1}^d\mathbb{E}\left[ (\ell_t^{(i)})^2(x_k^{(i)})^{1/q}
    \right]\\
& = q\, \mathbb{E}\left[  \sum_{\text{$s$ terms}}^{}(\ell_t^{(i)})^2(x_t^{(i)})^{1/q}\right] \\
&\leq qs(1/s)^{1/q}=qs^{1-1/q},
\end{align*}
where we used the assumption that $\ell_t$ has at most $s$ nonzero
components, and the fact that $x_t\in \Delta_d$. The first regret
bound is thus proven.
By choosing $\eta=\sqrt{\frac{2s^{1-1/q}}{(q-1)Td^{1/q}}}$, we balance both terms and get:
\[ R_T\leq 2q\sqrt{\frac{Td^{1/q}s^{1-1/q}}{2(q-1)}}=\sqrt{2q}\sqrt{Ts\left(\frac{d}{s}\right)^{1/q}\left( \frac{q}{q-1} \right) }\enspace. \]
If $d/s\geq e^2$ and $q=\log (d/s)$, then  $q/(q-1)\leq 2$
and finally:
\[ R_T\leq 2\sqrt{e}\sqrt{Ts\log \frac{d}{s}}. \]
\end{proof}
\begin{remark}\label{RM:upper_bound_bandit}
The previous analysis cannot be carried in the case of gains because
the bound from \cite[Theorem 5.10]{bubeck2012regret} that we use above only holds for
nonnegative losses (and its proof strongly relies on this
assumption). We are unaware of techniques which could provide a similar
bound in the case of nonnegative gains.
\end{remark}

\subsection{Matching Lower Bound}

The following theorem establishes that the bound from
Theorem~\ref{thm:bandit-upper bound-losses} 
is optimal up to a
logarithmic factor. 
We denote $\hat{v}_T^{\ell,s,d}$ the minimax regret
in the bandit setting with losses.

\begin{theorem}
\label{thm:bandit-losses-lowerbound}
For all $d\geqslant 2$, $s\in [d]$ and $T\geqslant d^2/4s$, the
following lower bound holds:
\[ \hat{v}_T^{\ell,s,d}\geqslant \frac{1}{32}\sqrt{Ts}. \]
\end{theorem}

The intuition behind the proof is the following. Let us consider the case where  $s=1$ and assume that $\omega_t$ is a  unit vector $e_{i_t} = (\mathbbm{1}\{j=i_t\})_j$  where $ \mathbb{P}(i_t=i)\simeq (1-\varepsilon)/d$ for all $i \in [d]$, except one fixed coordinate $i^*$ where $ \mathbb{P}(i_t=i^*)\simeq 1/d+\varepsilon$.
 
 Since $1/d$ goes to 0 as $d$ increases, the Kullback-Leibler divergence between two Bernoulli of parameters $(1-\varepsilon)/d$ and $1/d+\varepsilon$ is of order $d\varepsilon^2$. As a consequence, it would require approximately $1/d\varepsilon^2$ samples to distinguish between the two. The standard argument that one of the coordinates has not been chosen more than $T/d$ times, yields that one should take $1/d\varepsilon^2 \simeq T/d$ so that the regret is of order $T\varepsilon$. This  provides a lower bound of order $\sqrt{T}$. Similar arguments with $s>1$  give a lower bound of order $\sqrt{sT}$.

We emphasize that one cannot simply assume that the $s$
components with positive losses are chosen at the beginning once for
all, and apply standard lower bound techniques. Indeed, with
this additional information, the decision maker just has to choose, at
each stage, a decision associated with a zero loss. His regret
would then be uniformly bounded (or even possibly equal to zero).

\subsection{Proof of Theorem~\ref{thm:bandit-losses-lowerbound}}

Let $d\geqslant 1$, $1\leqslant s\leqslant d$, $T\geqslant 1$, and $\varepsilon\in (0,s/2d)$. Denote $\mathfrak{P}_s([d])$
the set of subsets of $[d]$ of cardinality $s$, $\delta_{ij}$ the
Kronecker symbol, and $B(1,p)$ the Bernoulli distribution of parameter
$p\in [0,1]$. If $P,Q$ are two probability distributions on the same
set, $D_{}\left( P \,\middle|\!\middle|\, Q \right)$ will denote the
relative entropy of $P$ and $Q$.

\subsubsection{Random $s$-sparse loss vectors $\ell_t$ and $\ell'_t$}
\label{sec:random-loss-vectors}
For $t\geqslant 1$, define the random $s$-sparse loss vectors $(\ell_t)_{t\geqslant 1}$ as follows.
Draw $Z$ uniformly from $[d]$. We will denote $\mathbb{P}_i\left[ \,\cdot\,
    \right]=\mathbb{P}\left[ \,\cdot\,\,\middle|\,Z=i \right]$ and
    $\mathbb{E}_i\left[ \,\cdot\, \right]=\mathbb{E}\left[
    \,\cdot\,\,\middle|\,Z=i \right]$.
 Knowing $Z=i$,  the random vectors $\ell_t$ are
    i.i.d and defined as follows. Draw $I_t$ uniformly from $\mathfrak{P}_s([d])$.
    If $j\in I_t$, define $\ell_t^{(j)}$ such that:
      \[ \mathbb{P}_i\left[ \ell_t^{(j)}=1 \right]=1-\mathbb{P}_i\left[
      \ell_t^{(j)}=0 \right]=\frac{1}{2}-\frac{\varepsilon
      d}{s}\delta_{ij}.   \]
If $j\not \in I_t$, set $\ell_t^{(j)}=0$. Therefore, one can check that for each component $j\in [d]$ and
      all $t\geqslant 1$,
      \[ \mathbb{E}_i\left[ \ell_t^{(j)} \right]=\frac{s}{2d}-\varepsilon \delta_{ij}.  \]
For $t\geqslant 1$, define the i.i.d. random $s$-sparse loss vectors $(\ell'_t)_{t\geqslant 1}$ as follows.
Draw $I'_t$ uniformly from $\mathfrak{P}_s([d])$. Then if $j\in I'_t$, set $(\ell'_t)^{(j)}$ such that:
    \[ \mathbb{P}\left[ (\ell'_t)^{(j)}=1 \right]=\mathbb{P}\left[ (\ell'_t)^{(j)}=0 \right]=1/2. \]
  And if $j\not \in I'_t$, set $(\ell'_t)^{(j)}=0$. Therefore, one can check that for each component $j\in [d]$ and
    all $t\geqslant 1$,
    \[ \mathbb{E}_i\left[ (\ell'_t)^{(j)} \right]=\frac{s}{2d}.  \]
By construction, $\ell_t$ and $\ell_t'$ are indeed random $s$-sparse loss vectors.

\subsubsection{A deterministic strategy $\sigma$ for the player}
\label{sec:determ-strat-sigma}

We assume given a deterministic strategy
$\sigma=(\sigma_t)_{t\geqslant 1}$ for the player:
\[ \sigma_t:([d]\times [0,1])^{t-1}\longrightarrow [d]. \]
Therefore,
\[ d_t=\sigma_t(d_1,\omega_1^{(d_1)},\dots,d_{t-1},\omega_{t-1}^{(d_{t-1})}), \]
where $d_t$ denotes the decision chosen by the strategy at stage $t$ and
$\omega_t$ the outcome vector of stage $t$. But since $d_t$ is determined by
previous decisions and outcomes, we can consider that $\sigma_t$ only
depends on the received outcomes:
\[ \sigma_t:[0,1]^{t-1}\longrightarrow [d], \]
\[ d_t=\sigma_t(\omega_1^{(d_1)},\dots,\omega_{t-1}^{(d_{t-1})}). \]

We define $d_t$ and $d'_t$ to be the (random) decisions played by
deterministic strategy $\sigma$ against the random loss vectors $(\ell_t)_{t\geqslant 1}$ and
$(\ell'_t)_{t\geqslant 1}$ respectively:
\begin{align*}
  d_t&=\sigma_t(\ell_1^{(d_1)},\dots,\ell_{t-1}^{(d_{t-1})}),\\
  d'_t&=\sigma_t((\ell'_1)^{(d'_1)},\dots,(\ell'_{t-1})^{(d'_{t-1})}).
\end{align*}

For $t\geqslant 1$ and $i\in [d]$, define $A_t^{(i)}$ to be the set of sequences of
outcomes in $\left\{ 0,1 \right\}$ of the first $t-1$ stages for
which strategy $\sigma$ plays decision $i$ at stage $t$:
\[ A_t^{(i)}=\left\{ (u_1,\dots,u_{t-1})\in \left\{ 0,1 \right\}^{t-1}\,\middle|\,\sigma_t(u_1,\dots,u_{t-1})=i  \right\},  \]
and $B_t^{(i)}$ the complement:
\[ B_t^{(i)}=\left\{0,1\right\}^{t-1}\setminus A_t^{(i)}. \]
Note that for a given $t\geqslant 1$, $(A_t^{(i)})_{i\in [d]}$ is a
partition of $\left\{ 0,1 \right\}^{t-1}$ (with possibly some empty sets).

For $i\in [d]$, define $\tau_i(T)$ (resp.  $\tau'_i(T)$) to be the number of
times decision $i$ is played by strategy $\sigma$ against loss vectors $(\ell_t)_{t\geqslant 1}$
 (resp. against $(\ell'_t)_{t\geqslant 1}$) between stages $1$ and $T$:
\[ \tau_i(T)=\sum_{t=1}^T\mathbbm{1}_{\left\{ d_t=i \right\} }\quad \text{and}\quad \tau'_i(T)=\sum_{t=1}^T\mathbbm{1}_{\left\{ d'_t=i \right\} }. \]

\subsubsection{The probability distributions $\mathbb Q$ and $\mathbb Q_i$
  ($i\in [d]$) on binary sequences}
\label{sec:prob-distr-mathbb}

We consider binary sequences $\vec{u}=(u_1,\dots,u_T)\in \left\{ 0,1 \right\}^T$.
We define $\mathbb Q$ and $\mathbb Q_i$ ($i\in [d]$) to be probability
distributions on $\left\{ 0,1 \right\}^T$ as follows:
\begin{align*}
\mathbb Q_i\left[ \vec{u} \right]&=\mathbb{P}_i\left[
                                   \ell_1^{(d_1)}=u_1,\dots,\ell_T^{(d_T)}=u_T
                                   \right],\\
  \mathbb Q\left[ \vec{u} \right]&=\mathbb{P}\left[ (\ell'_1)^{(d'_1)}=u_1,\dots,(\ell'_T)^{(d'_T)}=u_T \right].
\end{align*}

Fix $(u_1,\dots,u_{t-1})\in \left\{ 0,1 \right\}^t$. The applications
\[ u_t\longmapsto \mathbb Q\left[ u_t\,\middle|\,u_1,\dots,u_{t-1} \right]\quad \text{and}\quad u_t\longmapsto \mathbb Q_i\left[ u_t\,\middle|\,u_1,\dots,u_{t-1} \right],   \]
are probability distributions on $\left\{ 0,1 \right\}$, which we now
aim at identifying. The first one is Bernoulli of
parameter $s/2d$. Indeed,
\begin{align*}
\mathbb Q\left[ 1\,\middle|\,u_1,\dots,u_{t-1} \right]&=\mathbb{P}\left[ (\ell'_t)^{(d'_t)}=1\,\middle|\,(\ell'_1)^{(d'_1)}=u_1,\dots,(\ell'_{t-1})^{(d'_{t-1})}=u_{t-1} \right]\\
&=\mathbb{P}\left[ (\ell'_t)^{(d'_t)}=1 \right]\\
&=\mathbb{P}\left[ d'_t\in I'_t \right]\mathbb{P}\left[ (\ell'_t)^{(d_t)}=1\,\middle|\,d'_t\in I'_t \right]\\
&=\frac{s}{d}\times \frac{1}{2}\\
&=\frac{s}{2d},
\end{align*}
where we used the independence of the random vectors $(\ell'_t)_{t\geqslant 1}$ for
the second inequality. We now turn to the second distribution, which
depends on $(u_1,\dots,u_{t-1})$. If $(u_1,\dots,u_{t-1})\in A_t^{(i)}$,
it is a Bernoulli of parameter $s/2d-\varepsilon$:
\begin{align*}
\mathbb Q_i\left[ 1\,\middle|\,u_1,\dots,u_{t-1} \right]&=\mathbb{P}_i\left[ \ell_t^{(d_t)}=1\,\middle|\,\ell_1^{(d_1)}=u_1,\dots,\ell_{t-1}^{(d_{t-1})}=u_{t-1}\right]\\
&=\mathbb{P}_i\left[ \ell_t^{(i)}=1\,\middle|\,\ell_1^{(d_1)}=u_1,\dots,\ell_{t-1}^{(d_{t-1})}=u_{t-1} \right]\\
&=\mathbb{P}_i\left[ \ell_t^{(i)}=1 \right]\\
&=\mathbb{P}_i\left[ i\in I_t \right]\mathbb{P}_i\left[ \ell_t^{(i)}=1\,\middle|\,i\in I_t \right]\\
&=\frac{s}{d}\times \left( \frac{1}{2}-\frac{\varepsilon d}{s} \right)\\
&=\frac{s}{2d}-\varepsilon.
\end{align*}
where for the third inequality, we used the assumption that the
random vectors $(\ell_t)_{t\geqslant 1}$ are independent under $\mathbb P_i$, i.e. knowing
$Z=i$. On the other hand, if $(u_1,\dots,u_{t-1})\in B_t^{(i)}$, we can
prove similarly that the distribution is a Bernoulli of parameter $s/2d$.

\subsubsection{Computation the relative entropy of $\mathbb Q_i$ and
  $\mathbb Q$}
\label{sec:comp-relat-entr}
We apply iteratively the chain rule to the relative entropy of $\mathbb
Q[\vec{u}]$ and $\mathbb Q_i[\vec{u}]$. Using the short-hand $\mathbb D_i[\,\cdot\,]:=D_{}\left( \mathbb Q[\,\cdot\,] \,\middle|\!\middle|\, \mathbb Q_i[\,\cdot\,] \right)$,
\begin{align*}
D_{}\left( \mathbb Q\left[\vec{u}\right] \,\middle|\!\middle|\,
  \mathbb Q_i\left[\vec{u}\right] \right)&=\mathbb{D}_i[\vec{u}]\\
  &=\mathbb D_i\left[u_1\right]+\mathbb D_i\left[u_2,\dots,u_T \,\middle|\, u_1\right]\\
&=\mathbb D_i\left[u_1\right]+\mathbb D_i\left[u_2\,\middle|\,u_1\right]+\mathbb D_i\left[u_3,\dots,u_T\,\middle|\,u_1,u_2\right]\\
&=\sum_{t=1}^T\mathbb D_i\left[u_t\,\middle|\,u_1,\dots,u_{t-1}\right].
\end{align*}
We now use the definition of the conditional relative entropy, and
make the previously discussed Bernoulli distributions appear. For $1\leqslant t\leqslant T$,
\begin{align*}
\mathbb D_i&\left[ u_t\,\middle|\,u_1,\dots , u_{t-1}
  \right]=\sum_{u_1,\dots,u_{t-1}}^{}\mathbb Q\left[
           u_1,\dots,u_{t-1} \right]\\
&\qquad \qquad \qquad \qquad \qquad \qquad    \times \sum_{u_t}^{}\mathbb Q\left[ u_t\,\middle|\,u_1,\dots,u_{t-1} \right]\log \frac{\mathbb Q\left[ u_t\,\middle|\,u_1,\dots,u_{t-1} \right] }{\mathbb Q_i\left[ u_t\,\middle|\,u_1,\dots,u_{t-1} \right] }\\
&\quad \quad \quad =\frac{1}{2^{t-1}}\sum_{u_1,\dots,u_{t-1}}^{}\sum_{u_t}^{}\mathbb Q\left[ u_t\,\middle|\,u_1,\dots,u_{t-1} \right]\log \frac{\mathbb Q\left[ u_t\,\middle|\,u_1,\dots,u_{t-1} \right] }{\mathbb Q_i\left[ u_t\,\middle|\,u_1,\dots,u_{t-1} \right] }\\
&\quad \quad \quad =\frac{1}{2^{t-1}}\sum_{(u_1,\dots,u_{t-1})\in A_t^{(i)}}^{}D_{}\left( B\left(1,\frac{s}{2d}\right) \,\middle|\!\middle|\, B\left(1,\frac{s}{2d}-\varepsilon\right) \right)\\
&\quad \quad\quad \quad \quad \quad \quad  +\frac{1}{2^{t-1}}\sum_{(u_1,\dots,u_{t-1})\in B_t^{(i)}}^{}D_{}\left( B\left(1,\frac{s}{2d}\right) \,\middle|\!\middle|\, B\left(1,\frac{s}{2d}\right) \right)\\
&\quad \quad \quad =\frac{1}{2^{t-1}}\sum_{(u_1,\dots,u_{t-1})\in A_t^{(i)}}^{} \mathbb B\left( \frac{s}{2d},\varepsilon \right), 
\end{align*}
where we used the short-hand $\mathbb B\left( \frac{s}{2d},\varepsilon \right):= D\left(  B\left(1,\frac{s}{2d}\right) \,\middle|\!\middle|\, B\left(1,\frac{s}{2d}-\varepsilon\right)\right)$.
Eventually:
\[ 
D_{}\left( \mathbb Q[\vec{u}] \,\middle|\!\middle|\, \mathbb
  Q_i[\vec{u}] \right)  =\mathbb B\left(\frac{m}{2d},\varepsilon\right)\sum_{t=1}^T\frac{\left|  A_t^{(i)} \right|}{2^{t-1}}. \]

\subsubsection{Upper bound on
  $\frac{1}{d}\sum_{i=1}^{d}\mathbb{E}_i\left[ \tau_i(T) \right]$ using
  Pinsker's inequality}
\label{sec:upper-bound-frac1ds}
In this step, we will make use of Pinsker's inequality to make the
relative entropy appear.
\begin{proposition}[Pinsker's inequality]
Let $X$ be a finite set, and $P,Q$ probability distributions on $X$. Then,
\[ \frac{1}{2}\sum_{x\in X}^{}\left| P(x)-Q(x) \right| \leqslant \sqrt{\frac{1}{2}D_{}\left( P \,\middle|\!\middle|\, Q \right)}. \]
Immediate consequence:
\[ \sum_{\substack{x\in X\\P(x)>Q(x)}}^{}\left( P(x)-Q(x) \right)\leqslant \sqrt{\frac{1}{2}D_{}\left( P \,\middle|\!\middle|\, Q \right)}.  \]
\end{proposition}

Let $i\in [d]$. If $(u_1,\dots,u_T)\in \left\{ 0,1 \right\}^T$ is
given, since the decisions $d_t$ and $d'_t$ are determined by the previous
losses $\ell_t^{(d_t)}$ and $(\ell'_t)^{(d'_t)}$ respectively, we have in particular:
\[ \mathbb{E}_i\left[ \tau_i(T)\,\middle|\,\ell_1^{(d_1)}=u_1,\dots,\ell_T^{(d_T)}=u_T \right]=\mathbb{E}\left[ \tau_i'(T)\,\middle|\,(\ell'_1)^{(d'_1)}=u_1,\dots,(\ell'_T)^{(d'_T)} =u_T\right].   \]
Therefore,
\begin{align*}
\mathbb{E}_i\left[ \tau_i(T) \right]-\mathbb{E}\left[ \tau'_i(T)
  \right]&=\sum_{\vec{u}}^{}\mathbb Q_i[\vec{u}]\cdot
           \mathbb{E}_i\left[ \tau_i(T)\,\middle|\,\forall t,\
           \ell_t^{(d_t)}=u_t \right]\\
&\quad \quad\quad  -\sum_{\vec{u}}^{}\mathbb Q[\vec{u}]\cdot \mathbb{E}\left[ \tau'_i(T)\,\middle|\,\forall t,\ (\ell'_t)^{d'_t}=u_t \right]\\
&=\sum_{\vec{u}}^{}\left( \mathbb Q_i[\vec{u}]-\mathbb Q[\vec{u}] \right)\mathbb{E}_i\left[ \tau_i(T)\,\middle|\,\forall t,\ \ell_t^{(d_t)}=u_t \right]\\
&\leqslant \sum_{\substack{\vec{u}\\\mathbb Q_i[\vec{u}]>\mathbb Q[\vec{u}]}}^{}\left( \mathbb Q_i[\vec{u}]-\mathbb Q[\vec{u}] \right) \mathbb{E}_i\left[ \tau_i(T)\,\middle|\, \forall t,\ \ell_t^{(d_t)}=u_t \right]\\
&\leqslant T\sum_{\substack{\vec{u}\\\mathbb Q_i[\vec{u}]>\mathbb Q[\vec{u}]}}^{}\left( \mathbb Q_i[\vec{u}]-\mathbb Q[\vec{u}] \right)\\
&\leqslant T\sqrt{\frac{1}{2}D_{}\left( \mathbb Q[\vec{u}] \,\middle|\!\middle|\, \mathbb Q_i[\vec{u}] \right)}\\
&=T\sqrt{\frac{\mathbb B(s/2d,\varepsilon)}{2}}\sqrt{\sum_{t=1}^T\frac{\left| A_t^{(i)} \right| }{2^{t-1}}},\\
\end{align*}
where we used Pinsker's inequality in the fifth line. Moreover, we have:
\[ \frac{1}{d}\sum_{i=1}^d\mathbb{E}\left[ \tau'_i(T) \right]=\frac{1}{d}\mathbb{E}\left[ \sum_{t=1}^T\sum_{i=1}^d\mathbbm{1}_{\left\{ d'_t=i \right\} } \right]=\frac{1}{d}\mathbb{E}\left[ \sum_{t=1}^T1 \right]=\frac{T}{d}.     \]
Combining this with the previous inequality gives:
\begin{align*}
\frac{1}{d}\sum_{i=1}^d\mathbb{E}_i\left[ \tau_i(T) \right]&\leqslant \frac{1}{d}\sum_{i=1}^d\mathbb{E}\left[ \tau_i'(T) \right]+T\sqrt{\frac{\mathbb B(s/2d,\varepsilon)}{2}}\frac{1}{d}\sum_{i=1}^d\sqrt{\sum_{t=1}^T\frac{\left| A_t^{(i)} \right| }{2^{t-1}}}\\
&\leqslant \frac{T}{d}+T\sqrt{\frac{\mathbb B(s/2d,\varepsilon)}{2}}\sqrt{\frac{1}{d}\sum_{t=1}^{T}\sum_{i=1}^d\frac{\left| A_t^{(i)} \right| }{2^{t-1}}}\\
&=\frac{T}{d}+T\sqrt{\frac{\mathbb B(s/2d,\varepsilon)}{2}}\sqrt{\frac{1}{d}\sum_{t=1}^T\frac{\left| \left\{ 0,1 \right\}^{t-1}  \right| }{2^{t-1}}}\\
&=\frac{T}{d}+T\sqrt{\frac{\mathbb B(s/2d,\varepsilon)}{2}}\sqrt{\frac{T}{d}}\\
&=\frac{T}{d}+T^{3/2}\sqrt{\frac{\mathbb B(s/2d,\varepsilon)}{2d}}.
\end{align*}
where we used Jensen for the second inequality, and for the third
line, we remembered that $(A_t^{(i)})_{i\in [d]}$ is a partition of $\left\{ 0,1 \right\}^{t-1}$.

% \begin{align*}
% \mathbb B\left( s/2d,\varepsilon \right)&=D_{}\left( B(1,s/2d) \,\middle|\!\middle|\, B(1,s/2d-\varepsilon) \right)\\
% &=\frac{s}{2d}\log \frac{s/2d}{s/2d-\varepsilon}+\left( 1-\frac{s}{2d} \right)\log \frac{1-s/2d}{1-s/2d+\varepsilon}\\
% &=-\frac{s}{2d}\log \left( 1-\frac{2d\varepsilon}{s} \right)+\left( \frac{s}{2d}-1 \right)\log \left( 1+\frac{\varepsilon}{1-s/2d} \right)\\
% &=-\frac{s}{2d}\left( -\frac{2d\varepsilon}{s}-\frac{1}{2}\left( \frac{2d\varepsilon}{s} \right)^2+o(\varepsilon^2)  \right)+\left( \frac{s}{2d}-1 \right)\left( \frac{\varepsilon}{1-s/2d}-\frac{1}{2}\left( \frac{\varepsilon}{1-s/2d} \right)^2+o(\varepsilon^2)  \right)\\
% &=\varepsilon+\frac{d\varepsilon^2}{s}-\varepsilon+\frac{1}{2}\frac{\varepsilon^2}{1-s/2d}+o(\varepsilon^2)\\
% &=\varepsilon^2\left( \frac{d}{s}+\frac{1}{1-s/2d} \right)+o(\varepsilon^2)\\
% &=\frac{2d^2\varepsilon^2}{m(2d-m)}+o(\varepsilon^2).
% \end{align*}

\subsubsection{An upper bound on $\mathbb{B}(s/2d,\varepsilon)$ for small
  enough $\varepsilon$}
\label{sec:an-upper-bound}
We first write $\mathbb{B}(s/2d,\varepsilon)$ explicitely.
\begin{align*}
\mathbb{B}\left( \frac{s}{2d},\varepsilon \right)&=D_{}\left(B(1,s/2d) \,\middle|\!\middle|\, B(1,s/2d-\varepsilon) \right)\\
&=\frac{s}{2d}\log \frac{s/2d}{s/2d-\varepsilon}+\left( 1-\frac{s}{2d}
  \right)\log \frac{1-s/2d}{1-s/2d+\varepsilon}\\
&=-\frac{s}{2d}\log \left( 1-\frac{2d\varepsilon}{s} \right)+\left( \frac{s}{2d}-1 \right)\log \left( 1+\frac{\varepsilon}{1-m/2d} \right).     
\end{align*}
We now bound the two logarithms from above using respectively the two
following easy inequalities:
\begin{align*}
-\log (1-x)&\leqslant x+x^2,\quad \text{for $x\in [0,1/2]$}\\
-\log (1+x)&\leqslant -x+x^2,\quad \text{for $x\geqslant 0$}.
\end{align*}
This gives:
\begin{align*}
\mathbb{B}\left( \frac{s}{2d},\varepsilon \right)&\leqslant
                                                   \frac{s}{2d}\left(
                                                   \frac{2d\varepsilon}{s}+\frac{4d^2\varepsilon^2}{s^2}
                                                   \right)+\left(
                                                   1-\frac{s}{2d}
                                                   \right)\left(
                                                   -\frac{\varepsilon}{1-s/2d}+\frac{\varepsilon^2}{(1-s/2d)^2}
                                                   \right)\\
&=\frac{4d^2\varepsilon^2}{s(2d-s)},    
\end{align*}
which holds for $2d\varepsilon/s\leqslant 1/2$, in other words, for $\varepsilon\leqslant s/4d$.

\subsubsection{Lower bound on the expectation of the regret of $\sigma$
  against $\ell_t$}
\label{sec:lower-bound-expect}

We can now bound from below the expected regret incurred when playing
$\sigma$ against loss vectors $(\ell_t)_{t\geqslant 1}$. For $\varepsilon\leqslant s/4d$,
\begin{align*}
  R_T&= \mathbb{E}\left[ \sum_{t=1}^T\ell_t^{(d_t)}-\min_{j\in
       [d]}\sum_{t=1}^T\ell_t^{(j)} \right]\\
&=
                                        \frac{1}{d}\sum_{i=1}^d\mathbb{E}_i\left[
      \sum_{t=1}^{T}\ell_t^{(d_t)}-\min_{j\in
                                        [d]}\sum_{t=1}^{T}\ell_t^{(j)}
                                        \right]\\
  &\geqslant \frac{1}{d}\sum_{i=1}^d\left( \mathbb{E}_i\left[ \sum_{t=1}^T\ell_t^{(d_t)} \right]-\min_{j\in
                                        [d]}\sum_{t=1}^T\mathbb{E}_i\left[
    \ell_t^{(j)} \right]   \right)\\
  &=\frac{1}{d}\sum_{i=1}^d\left( \mathbb{E}_i\left[ \sum_{t=1}^T\mathbb{E}_i\left[ \ell_t^{(d_t)}\,\middle|\,d_t \right]  \right]-T\min_{j\in
                                        [d]}\left(
    \frac{s}{2d}-\varepsilon \delta_{ij} \right)   \right)\\
  &=\frac{1}{d}\sum_{i=1}^d\left( \mathbb{E}_i\left[
    \sum_{t=1}^T\left( \frac{s}{2d}-\varepsilon \delta_{id_t} \right)
    \right]-T\left( \frac{s}{2d}-\varepsilon \right)   \right)\\
  &=\frac{1}{d}\sum_{i=1}^d\varepsilon\left( T-\mathbb{E}_i\left[
    \tau_i(T) \right]  \right)\\
&=\varepsilon\left( T-\frac{1}{d}\sum_i^{}\mathbb{E}_i\left[ \tau_i(T) \right]  \right).
\end{align*}
We now use the upper bound derived in Section~\ref{sec:upper-bound-frac1ds}.
\begin{align*}
R_T&\geqslant \varepsilon\left( T-\frac{T}{d}-T^{3/2}\sqrt{\frac{\mathbb
  B(s/2d,\varepsilon)}{2d}} \right)\\
&\geqslant \varepsilon\left(
  T-\frac{T}{d}-T^{3/2}\varepsilon\sqrt{\frac{2d}{s(2d-s)}} \right)\\
&\geqslant \varepsilon \left( T-\frac{T}{d}-2T^{3/2}\varepsilon\frac{1}{\sqrt{s}}. \right) , 
\end{align*}
where in the penultimate, we used the upper bound on
$\mathbb{B}(s/2d,\varepsilon)$ that we established above, and in the
last line, the fact that $s\leqslant d$. Let $C>0$
and we choose $\varepsilon=C\sqrt{s/T}$. Then, for
$\varepsilon\leqslant s/4d$,
\begin{align*}
R_T&\geqslant \varepsilon T\left( 1-\frac{1}{d}-2\varepsilon\sqrt{\frac{T}{s}}
     \right)\\
&=C\sqrt{sT}\left( 1-\frac{1}{d} \right)-2\sqrt{sT}C^2\\
&\geqslant \sqrt{sT}\left( \frac{C}{2}-2C^2 \right),
\end{align*}
where in the last line, we used the assumption $d\geqslant 2$. The
choice $C=1/8$ give:
\[ R_T\geqslant \frac{1}{32}\sqrt{sT}, \]
which holds for $\varepsilon=C\sqrt{s/T}\leqslant s/4d$ i.e. for $T\geqslant d^2/4s$.

The above inequality does not depend on $\sigma$.
As it is a classic that a randomized strategy is equivalent to some random choice of deterministic
strategies, this lower bound  holds for any strategy of
the player. In other words, for $T\geqslant d^2/4s$,
\[ \hat{v}_T^{\ell,s,d}\geqslant \frac{1}{32}\sqrt{sT}. \] 
\qed
\subsection{Discussion}
\label{sec:discussion}

If the outcomes are not losses but gains, then there is an important
discrepancy between the upper and lower bounds we obtain. Indeed,
obtaining small losses regret bound as in the first displayed equation
of the proof of Theorem \ref{thm:bandit-upper bound-losses} is still
open. An idea for circumventing this issue would be to enforce exploration
by perturbing $x_t$ into $(1-\gamma)x_t + \gamma \mathcal{U}$ where
$\mathcal{U}$ is the uniform distribution over $[d]$, but usual
computations show that the only obtainable upper bounds are of 
order of $\sqrt{dT}$. The aforementioned techniques used to bound the
regret from below with losses would also work with gains, which would
give a lower bound of order $\sqrt{sT}$. Therefore, finding the optimal
dependency in the dimension and/or the sparsity level is still an open question in that specific case.

%This again advocates for the existence of a crucial discrepancy between the two models of regret minimization, whether they are losses to minimize or gains to maximize.
%
%\section{Conclusion}
%\label{sec:conclusion}
%The results, as we presented them, only hold in expectation and when
%the horizon $T$ (i.e., the total number of stages) is known in
%advance. However, using well-known techniques such as the ``doubling
%trick'' or a more elegant time-dependent Online Mirror Descent (see
%e.g.  \cite{cesa2006prediction} or \cite{kwon2014continuous}), one
%could prove and state anytime versions of the main theorems. The only
%difference would be a change in the multiplicative constants.
%Similarly, we could establish the same results with high probability guarantees instead of in expectation, using again well-known techniques \cite{jake2009}.
%
%As mentioned before and suggested by the table in
%Figure~\ref{fig:tableau}, the more challenging and open question is whether
%we could obtain, in the bandit framework with gains, an upper bound of order $\sqrt{Ts}$ or a lower bound of order $\sqrt{Td}$.

\footnotesize 
\bibliographystyle{siam}
\bibliography{sparse}

\begin{thebibliography}{10}

\bibitem{abbasi2012online}
{\sc Y.~Abbasi-Yadkori, D.~Pal, and C.~Szepesvari}, {\em
  Online-to-confidence-set conversions and application to sparse stochastic
  bandits}, in AISTATS, vol.~22, 2012, pp.~1--9.

\bibitem{audibert2009minimax}
{\sc J.-Y. Audibert and S.~Bubeck}, {\em Minimax policies for adversarial and
  stochastic bandits}, in Proceedings of the Annual Conference on Learning
  Theory (COLT), 2009, pp.~217--226.

\bibitem{audibert2013regret}
{\sc J.-Y. Audibert, S.~Bubeck, and G.~Lugosi}, {\em Regret in online
  combinatorial optimization}, Mathematics of Operations Research, 39 (2013),
  pp.~31--45.

\bibitem{auer2002adaptive}
{\sc P.~Auer, N.~Cesa-Bianchi, and C.~Gentile}, {\em Adaptive and
  self-confident on-line learning algorithms}, Journal of Computer and System
  Sciences, 64 (2002), pp.~48--75.

\bibitem{bubeck2011introduction}
{\sc S.~Bubeck}, {\em Introduction to online optimization}, Princeton
  University, 2011.

\bibitem{bubeck2012regret}
{\sc S.~Bubeck and N.~Cesa-Bianchi}, {\em Regret analysis of stochastic and
  nonstochastic multi-armed bandit problems}, Machine Learning, 5 (2012),
  pp.~1--122.

\bibitem{carpentier2012bandit}
{\sc A.~Carpentier and R.~Munos}, {\em Bandit theory meets compressed sensing
  for high dimensional stochastic linear bandit}, in International Conference
  on Artificial Intelligence and Statistics, 2012, pp.~190--198.

\bibitem{cesa1997analysis}
{\sc N.~Cesa-Bianchi}, {\em Analysis of two gradient-based algorithms for
  on-line regression}, in Proceedings of the Tenth Annual Conference on
  Computational Learning Theory, 1997, pp.~163--170.

\bibitem{cesa1997use}
{\sc N.~Cesa-Bianchi, Y.~Freund, D.~Haussler, D.~P. Helmbold, R.~E. Schapire,
  and M.~K. Warmuth}, {\em How to use expert advice}, Journal of the ACM, 44
  (1997), pp.~427--485.

\bibitem{cesa2006prediction}
{\sc N.~Cesa-Bianchi and G.~Lugosi}, {\em Prediction, learning, and games},
  Cambridge University Press, 2006.

\bibitem{djolonga2013high}
{\sc J.~Djolonga, A.~Krause, and V.~Cevher}, {\em High-dimensional gaussian
  process bandits}, in Advances in Neural Information Processing Systems, 2013,
  pp.~1025--1033.

\bibitem{freund1997decision}
{\sc Y.~Freund and R.~E. Schapire}, {\em A decision-theoretic generalization of
  on-line learning and an application to boosting}, Journal of computer and
  system sciences, 55 (1997), pp.~119--139.

\bibitem{galambos1980asymptotic}
{\sc J.~Galambos}, {\em The asymptotic theory of extreme order statistics},
  John Wiley, New York, 1978.

\bibitem{gerchinovitz2013sparsity}
{\sc S.~Gerchinovitz}, {\em Sparsity regret bounds for individual sequences in
  online linear regression}, The Journal of Machine Learning Research, 14
  (2013), pp.~729--769.

\bibitem{hannan1957approximation}
{\sc J.~Hannan}, {\em Approximation to bayes risk in repeated play},
  Contributions to the Theory of Games, 3 (1957), pp.~97--139.

\bibitem{hazan201210}
{\sc E.~Hazan}, {\em The convex optimization approach to regret minimization},
  Optimization for machine learning,  (2012), pp.~287--303.

\bibitem{kakade2012regularization}
{\sc S.~M. Kakade, S.~Shalev-Shwartz, and A.~Tewari}, {\em Regularization
  techniques for learning with matrices}, The Journal of Machine Learning
  Research, 13 (2012), pp.~1865--1890.

\bibitem{kwon2014continuous}
{\sc J.~Kwon and P.~Mertikopoulos}, {\em A continuous-time approach to online
  optimization}, arXiv preprint arXiv:1401.6956,  (2014).

\bibitem{littlestone1994weighted}
{\sc N.~Littlestone and M.~K. Warmuth}, {\em The weighted majority algorithm},
  Information and computation, 108 (1994), pp.~212--261.

\bibitem{rakhlinlecture}
{\sc A.~Rakhlin and A.~Tewari}, {\em Lecture notes on online learning},
  University of Pennsylvania, 2008.

\bibitem{shalev2007online}
{\sc S.~Shalev-Shwartz}, {\em Online learning: Theory, algorithms, and
  applications}, PhD Thesis, The Hebrew University of Jerusalem, 2007.

\bibitem{shalev2011online}
\leavevmode\vrule height 2pt depth -1.6pt width 23pt, {\em Online learning and
  online convex optimization}, Foundations and Trends in Machine Learning, 4
  (2011), pp.~107--194.

\bibitem{slepian1962one}
{\sc D.~Slepian}, {\em The one-sided barrier problem for gaussian noise}, Bell
  System Technical Journal, 41 (1962), pp.~463--501.

\bibitem{vovk1990aggregating}
{\sc V.~G. Vovk}, {\em Aggregating strategies}, in Proceedings of the Third
  Workshop on Computational Learning Theory, 1990, pp.~371--383.

\end{thebibliography}
%\endgroup

% \normalsize
% \bigskip
% \bigskip
% \bigskip

% \begin{center}
% \textsc{Joon Kwon\\[8pt]
% Institut de math\'{e}matiques de Jussieu\\
% Universit\'{e} Pierre-et-Marie-Curie\\
% Paris, France\\[8pt]
% e-mail: \textnormal{\texttt{joon.kwon@ens-lyon.org}}}
% \end{center}
% \bigskip
% \bigskip
% \bigskip
% \begin{center}
%   \textsc{Vianney Perchet\\[8pt]
%     INRIA \&\\
%     Laboratoire et probabilit\'{e}s et modèles al\'{e}atoires\\
%     Universit\'{e} Paris-Diderot\\
%     Paris, France\\[8pt]
%     e-mail: \textnormal{\texttt{vianney.perchet@normalesup.org}}}
% \end{center}

\end{document}